\documentclass[letterpaper, 10 pt, conference]{ieeeconf}  

\IEEEoverridecommandlockouts                              

\usepackage{graphics} 
\usepackage{epsfig} 
\usepackage{mathptmx} 
\usepackage{times} 
\usepackage{amsmath} 
\usepackage{amssymb}  

\usepackage{amsthm}
\newcommand{\minus}{\scalebox{0.8}[1.0]{$-$}}
\newcommand{\plus}{\scalebox{0.8}[1.0]{$+$}}
\newcommand{\equals}{\scalebox{0.8}[1.0]{$=$}}

\usepackage{bm}
\usepackage{cite}
\usepackage{xcolor,graphicx}
\usepackage{algorithm}
\usepackage{algpseudocode}
\usepackage{algorithmicx}
\usepackage{subcaption}
\usepackage{overpic}
\usepackage{booktabs}
\usepackage[english]{babel}
\usepackage{url}
\usepackage{placeins}  
\usepackage{afterpage}  
\usepackage{booktabs}  
\usepackage{multirow}  

\usepackage{balance}

\newtheorem{proposition}{Proposition}

\usepackage[capitalize,nameinlink]{cleveref}

\crefname{equation}{Eq.}{Eqs.}
\Crefname{equation}{Eq.}{Eqs.}
\usepackage{tikz}
\usepackage{xcolor}

\usepackage{etoolbox}

\let\labelindent\relax
\usepackage[inline]{enumitem}

\newcommand{\highlight}[1]{#1}

\AtBeginEnvironment{algorithm}{%
  \setlength{\textfloatsep}{5pt plus 1pt minus 2pt}
  \setlength{\intextsep}{5pt plus 1pt minus 2pt}
}

\title{\LARGE \bf
Self-Supervised Physics-Informed Manipulation of Deformable Linear Objects with Non-negligible Dynamics
}

\author{Youyuan Long$^{1, 2}$, Gokhan Solak$^{1}$, Sara Zeynalpour$^{1}$, Heng Zhang$^{1, 2}$, and Arash Ajoudani$^{1}$   
\thanks{This work was supported by the Horizon Europe Project TORNADO (GA 101189557).
}
\thanks{$^{1}$Human-Robot Interfaces and Interaction Lab, Istituto Italiano di Tecnologia, Genova, Italy
        {(e-mail: youyuan.long@iit.it; gokhan.solak@iit.it; sara.zeynalpour@iit.it; heng.zhang@iit.it; arash.ajoudani@iit.it).}}%
\thanks{$^{2}$Ph.D. program of national interest in Robotics and Intelligent Machines (DRIM) and Universita di Genova, Genoa, Italy.}%
}

\begin{document}

\maketitle
\thispagestyle{empty}
\pagestyle{empty}

\begin{abstract}
We address dynamic manipulation of deformable linear objects by presenting \textit{SPiD}, a physics-informed self-supervised learning framework that couples an accurate deformable object model with an augmented self-supervised training strategy.
On the modeling side, we extend a mass–spring model to more accurately capture object dynamics while remaining lightweight enough for high-throughput rollouts during self-supervised learning. 
On the learning side, we train a neural controller using a task-oriented cost, enabling end-to-end optimization through interaction with the differentiable object model. In addition, we propose a self-supervised DAgger variant that detects distribution shift during deployment and performs offline self-correction to further enhance robustness without expert supervision.
We evaluate our method primarily on the rope stabilization task, where a robot must bring a swinging rope to rest as quickly and smoothly as possible. 
Extensive experiments in both simulation and the real world demonstrate that the proposed controller achieves fast and smooth rope stabilization, generalizing across unseen initial states, rope lengths, masses, non-uniform mass distributions, and external disturbances. 
Additionally, as a secondary practical extension, we evaluate our method with a single-camera markerless rope perception method and demonstrate that our controller maintains performance with noisy and low-frequency state updates. 
Furthermore, we demonstrate the generality of the framework by extending it to the rope trajectory tracking task. 
Overall, SPiD offers a data-efficient, robust, and physically grounded framework for dynamic manipulation of deformable linear objects, featuring strong sim-to-real generalization.

\end{abstract}

\section{INTRODUCTION}

With the 
advances in robot autonomy and hardware,
robots are now being deployed in human-centered environments such as homes, shopping malls, and hospitals. These environments contain a wide variety of deformable objects, such as ropes, tubes, clothes, paper, and sponges~\cite{yin2021modeling, chi2024iterative}. 
Among them, deformable linear objects (DLOs), including ropes, cables, and flexible tubes, are particularly common in domestic and industrial applications.
The ability to manipulate such objects efficiently is therefore of great practical importance. However, tasks that are trivial for humans remain highly challenging for robots, particularly when fast motions and significant deformations are involved. 

Unlike rigid bodies, DLOs exhibit infinite-dimensional degrees of freedom and can easily deform under even small external forces or disturbances. Moreover, their deformation behavior is highly nonlinear and often involves coupled physical effects such as tension, bending, and twisting~\cite{yin2021modeling}. Consequently, accurately modeling their dynamics is significantly more complex than that of rigid-body systems.

    Furthermore, in contrast to quasi-static manipulation, dynamic manipulation (high-velocity actions) refers to control and interaction scenarios in which inertial and time-varying effects play a dominant role in the object’s motion~\cite{mason1993dynamic}, often resulting in fast and pronounced deformations for deformable objects.
    Such manipulation is highly challenging due to the complex dynamics of deformable objects, strict control latency requirements, and a significant sim-to-real gap.
    As a result, many traditional optimization-based approaches struggle to achieve real-time performance and stability~\cite{yin2021modeling}.

    Due to these challenges, many researchers have turned to data-driven approaches for deformable object manipulation, such as learning dynamic models from visual data \cite{zhang2025particle,shen2022acid} or training controllers from human demonstrations \cite{salhotra2022learning}. However, these methods typically require large amounts of high-quality data and often suffer from limited generalization capability \cite{mamedov2025learning}.

To address these challenges, we propose the \textbf{SPiD} (\textbf{S}elf-supervised \textbf{P}hysics-\textbf{i}nformed learning for \textbf{D}eformable object manipulation) framework, which consists of two main components: physics-informed modeling and learning-based control.
In our study, we focus on rope manipulation, which captures many of the core difficulties of deformable interaction, such as high-dimensional continuous state spaces, nonlinear dynamics, and strong coupling between contact forces and global object configuration~\cite{caporali2026robotic}.
In SPiD, \textit{a differentiable rope dynamics model} lies at the core of the framework. 
The model parameters are identified using differentiable physics–based system identification (Sec. \ref{sec:identification}).
This model is computationally efficient yet remains highly nonlinear and high-dimensional (in our simulation, the rope consists of 21 mass points, resulting in a 126-dimensional state space).
However, online optimization methods such as MPC can incur substantial computational costs due to repeated trajectory evaluations, particularly for high-dimensional deformable dynamics \cite{rawlings2020model, mayne2000nonlinear}.
Instead, we leverage the differentiable physics model to \textit{train a neural controller offline through self-supervision}, shifting the heavy computation to the training phase while enabling fast online control execution via efficient neural network matrix operations.
This framework highlights two key features:

\begin{figure*}[t]
    \centering
    \includegraphics[width=1.0\textwidth]{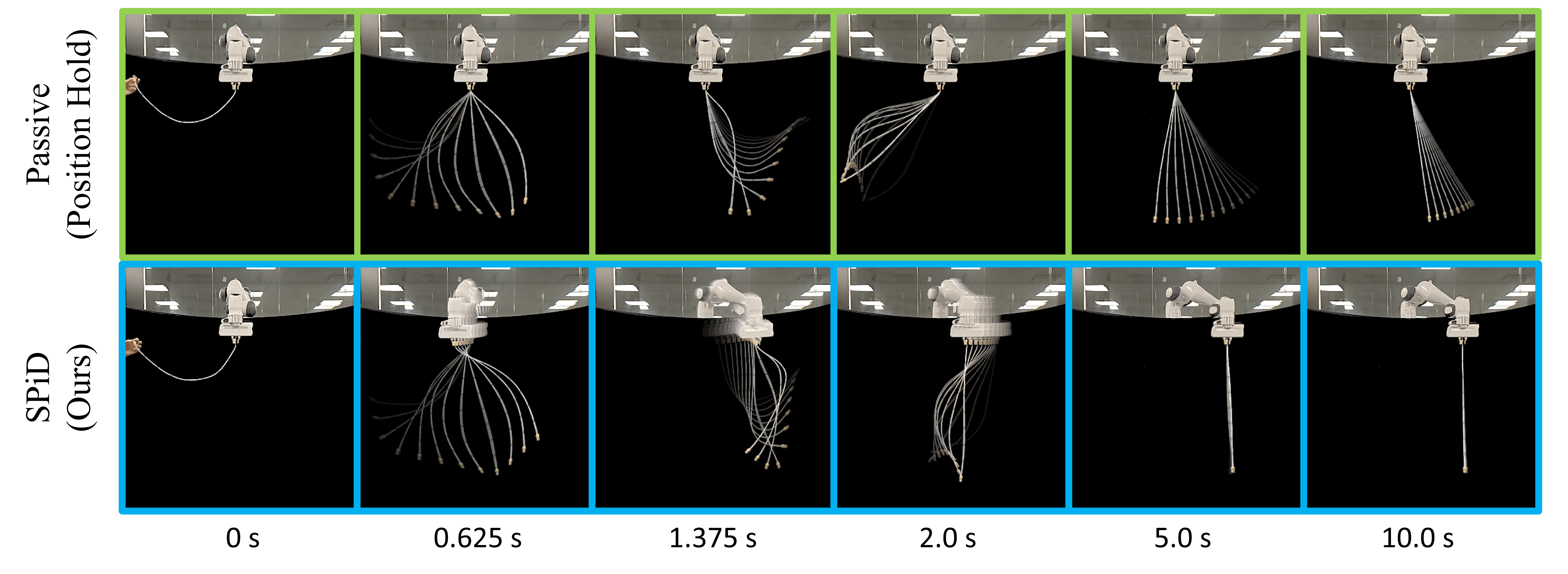}
    \vspace{-0.7cm}
    \caption{\textbf{Dynamic manipulation for rope stabilization.} 
Comparison between our controller and the passive case (where the robot arm holds its position). 
Each subfigure shows the rope motion at the current time, with 12 ghosted frames traced backward at 37.5~ms intervals—the fainter the trace, the earlier the time. 
Our controller actively computes actions that minimize the rope’s energy in real time, achieving superior performance, whereas the passive case relies solely on air resistance to dissipate energy. 
}
    \label{fig:time-lapse-real}
    \vspace{-0.4cm}
\end{figure*}

\begin{itemize}[leftmargin=*, topsep=0pt, itemsep=0pt, parsep=0pt, partopsep=0pt]
    \item \textbf{Dynamic physical modeling:} 
    To capture the dynamic characteristics of rope motion, we extend the conventional mass–spring model by introducing additional damping terms and \highlight{deriving their closed-form expressions.}
    Our experiments (Sec.~\ref{sec:exp-model-val}) demonstrate that our model more accurately captures the rope’s dynamic behavior than the previous approach in~\cite{lv2017physically}.

    \item \textbf{Augmented self-supervised learning strategy:} 
    By designing a task-oriented loss function and allowing the neural controller to interact with the deformable linear object (DLO) model, we train a control policy by minimizing this loss through gradient-based optimization.
    This differentiable-physics approach eliminates the need for expert demonstrations and shifts the major computational burden to the offline training phase.
    To mitigate the sim-to-real gap, we inject noise into the identified model parameters, enabling the controller to achieve robust performance across objects with different physical properties.
    Furthermore, we propose a \textit{self-supervised DAgger} mechanism (Sec.~\ref{our controller}) to address the distribution-shift problem in the self-supervised setting.
    This mechanism enables the controller to continually refine its policy under real-world feedback without relying on expert demonstrations.
\end{itemize}

To validate our approach, we focus on the rope stabilization task, which involves controlling one end of the rope to bring an initially swinging rope to rest as quickly as possible. This task is highly dynamic and practically significant in various applications, such as industrial hoisting operations, aerial hoist rescue, and UAV payload delivery.
However, it remains an underexplored problem. Most existing studies have focused on crane systems \cite{chen2017swing,zhao2024robust,liu2021collaborative}, cable-suspended UAV systems \cite{palunko2012trajectory,wu2024learning,omar2022integrating}, or elevator rope sway control~\cite{benosman2014lyapunov}.
In these works, the rope is often simplified as a single- or double-pendulum model with restricted swing angles, and in some cases, the rope motion is further constrained to a two-dimensional plane. Although such simplifications may be reasonable for specific scenarios, they significantly weaken the inherent three-dimensional flexibility and dynamic complexity of the rope as a deformable body.
In contrast, in our task, the rope is initialized at a large swing angle, exhibiting fast and rich dynamic behaviors during its return phase, and the entire process occurs in three-dimensional space (Fig.~\ref{fig:time-lapse-real}).

Our framework is validated in four perspectives: \begin{enumerate*}
    \item The \textit{physics modeling} contribution is evaluated in comparison to the closest modeling baseline \cite{lv2017physically}, demonstrating that it significantly improves the accuracy of long horizon prediction of rope motion.
    \item The \textit{learning and control} aspect is evaluated through the rope stabilization experiments, achieving agile manipulation with superior performance in comparison to the baseline \cite{chen2017swing}.
    \item The \textit{robustness} aspect is validated through rope stabilization using a single-camera markerless perception system. \highlight{As a secondary practical contribution}, we developed a data-driven markerless perception method and analyzed its limitations by comparison with a high-precision, marker-based motion capture system.
   In this setting, the single-camera setup provides planar observations, and SPiD remains effective despite low-frequency and noisy perception.
    \item \textit{Generalization} aspect is validated by applying SPiD on the rope trajectory tracking task that has different cost function and goal. SPiD successfully generalized to this task, achieving high-speed trajectory following in simulation.
\end{enumerate*}

In summary, our main contributions are as follows:
\begin{itemize}
\item SPiD, the novel physics-informed self-supervised learning framework that consists of a physically accurate DLO model and the augmented self-supervised training strategy incorporating a DAgger variant. The framework is efficient, robust, and generalizable to new tasks and objects as demonstrated by the experiments.
\item Extension of the mass-spring-based deformable object modeling by deriving damping terms that are proven valuable in dynamic motion prediction.
\item Achieving the rope stabilization task by energy-based task formulation without expert demonstrations, with high data-efficiency, and robustness to perturbations and noisy observations. To demonstrate the generalization capability of our framework, we also address a dynamic rope endpoint trajectory tracking task.
\item To further evaluate the practical robustness of SPiD, we demonstrate rope stabilization using a low-cost, single-camera markerless perception system. The experiments show that the learned controller maintains effective performance under noisy and low-frequency state observations, without additional training.
\end{itemize}

\section{Related work}
This section reviews relevant work on modeling, manipulation and perception of deformable objects.
\subsection{Physical modeling}
Previous studies on deformable object modeling have mainly relied on several mainstream physical formulations, including mass–spring systems (MSSs)~\cite{haumann1988behavioral,lv2017physically,zhong2024reconstruction}, position-based dynamics (PBD)~\cite{macklin2014unified,guler2015estimating}, finite element methods (FEM) \cite{greco2013b}, and elastic rod models \cite{bergou2008discrete}.
MSSs have been widely adopted due to their high computational efficiency, but their accuracy is limited because of model simplifications~\cite{lv2020review}.
PBD also provides fast simulation and good numerical stability, but focuses mainly on visual fidelity only~\cite{yin2021modeling}.
\highlight{FEM achieves high precision in modeling material deformation, but it usually requires substantial computational resources, which can limit its direct use in real-time applications without specialized acceleration or model reduction \cite{arriola2020modeling, gu2026survey}.
Elastic rod models are specifically designed for deformable linear object (DLO) modeling and provide a physically accurate description, but they similarly incur high computational costs \cite{orekhov2020solving}.
More recently, a soft-robotics-inspired strain-parameterized model has been proposed for dynamic DLO manipulation, enabling shape-regulation-based end-point positioning of slender deformable objects~\cite{tiburzio2025controlling}.}
In this work, we adopt a mass–spring formulation augmented with derived damping terms to retain computational efficiency while capturing key dynamic effects.

\subsection{Data-driven modeling}
Data-driven approaches aim to learn deformable object dynamics directly from sensory observations, often via neural transition models over particles, keypoints, or graphs.
Beyond earlier purely data-driven pipelines such as ACID~\cite{shen2022acid} and local-graph GNN models for deformable rearrangement~\cite{deng2023learning}, recent work has pushed toward higher-fidelity and more generalizable neural dynamics.
For instance, Particle-Grid Neural Dynamics~\cite{zhang2025particle} leverages particle--grid representations to better capture complex deformation and contact; AdaptiGraph~\cite{zhang2024adaptigraph} learns graph-based neural dynamics conditioned on continuous physical property variables and adapts to unseen deformable objects via few-shot interaction.

These methods model complex nonlinear behaviors that are difficult to capture analytically. 
However, reliable generalization of data-driven dynamics models often requires sufficiently diverse and high-quality training data \cite{gu2026survey}.
In contrast, our approach leverages a structured physical model with high data efficiency (using only 2,000 sampled data points in real-world experiments, as shown in Sec.~\ref{sec:exp-model-val}) and preserves physical interpretability.

\subsection{Quasi-static manipulation}
Quasi-static manipulation assumes that deformable objects remain near force equilibrium during slow interactions, allowing inertial and dynamic effects to be neglected.
This assumption has enabled a wide range of successful manipulation strategies by simplifying the underlying dynamics and reasoning primarily over object geometry.
\highlight{Representative works include DLO shape control~\cite{yu2022global,caporali2024deformable}, rope knotting~\cite{sundaresan2020learning}, constrained-space threading~\cite{peringal2026coordinated}, and cloth folding and smoothing~\cite{hoque2020visuospatial}.}

Earlier approaches often relied on demonstration transfer or hand-designed geometric heuristics to generalize manipulation strategies to new configurations~\cite{schulman2016learning}.
More recent works incorporate visual feedback, latent shape representations, or learning-based components to improve robustness and generalization under partial observability~\cite{lippi2020latent,yin2021modeling}.
These methods have demonstrated strong performance in tasks where object motion is slow and deformation evolves gradually.
However, despite their effectiveness, quasi-static approaches infer the effects of actions primarily from object geometry, which often leads to slow and inefficient execution due to the need for iterative correction~\cite{chi2024iterative}.


\subsection{Dynamic manipulation}
Beyond quasi-static assumptions, dynamic manipulation explicitly considers inertial and transient effects during deformable object interaction, enabling faster and more agile behaviors.
However, modeling and controlling deformable objects in dynamic regimes remains challenging due to complex, high-dimensional dynamics and limited observability.

Recent work has explored learning-based approaches to dynamic deformable manipulation.
Using state diffusion and inverse dynamics models, Chen et al.~\cite{chen2025learning} train policies capable of performing highly dynamic tasks such as laundry cleanup.
Similarly, Lan et al.~\cite{lan2025dynamic} propose the Dynamics-Informed Diffusion Policy, which integrates imitation pretraining with physics-informed test-time adaptation, and apply it to goal-conditioned dynamic rope manipulation.
While these methods demonstrate strong performance, they rely on a large amount of expert demonstrations, which limits scalability to new tasks and objects.
To reduce dependence on expert data, Lim et al.~\cite{lim2022real2sim2real} propose collecting large-scale simulated trajectories using a simulator calibrated with limited real-world observations.
However, this approach requires re-identification of physical parameters whenever the manipulated object changes, hindering robustness across objects with varying properties.
In contrast, our method achieves data-efficient dynamic manipulation without expert demonstrations and generalizes robustly to deformable objects with diverse physical parameters.

\subsection{Rope perception} 
Despite extensive research on deformable object manipulation, robust rope perception in realistic and dynamic environments remains an open problem \cite{nair2017combining}. Many existing methods rely on marker-based perception systems or assume static rope configurations and simplified geometric models, enabling accurate state estimation only under controlled conditions. Such approaches are difficult to deploy in unstructured settings. Vision-based methods also struggle with thin and elongated structures such as ropes due to frequent self-occlusions, rapidly changing shapes, and sensitivity to background clutter and illumination \cite{schulman2013tracking}. Recently, large pre-trained foundation models for segmentation have emerged as a promising alternative \cite{Kirillov_2023_ICCV}; however, their performance degrades for highly deformable objects without task-specific fine-tuning, and their computational cost limits real-time applicability. \highlight{Motivated by these limitations, we collect a task-specific dataset and fine-tune a lightweight, unified detection and segmentation model for low-cost markerless rope perception, which is used in this work as a secondary practical extension and evaluated in a planar setting.}

\section{SPiD: Self-supervised Physics-informed Deformable Object Manipulation}

The SPiD framework consists of an offline phase and an online phase as illustrated in Fig.~\ref{fig:framework}.
The offline phase includes deformable linear object (DLO) modeling and system identification, as well as a self-supervised control learning module.
The online phase deploys the learned neural policy for real-time control, while continuously performing out-of-distribution (OOD) detection, and uses a self-supervised DAgger mechanism to return to the control learning module to refine the controller and achieve more robust performance.

\begin{figure}[!b]
    \centering
    \includegraphics[width=1.0\linewidth]{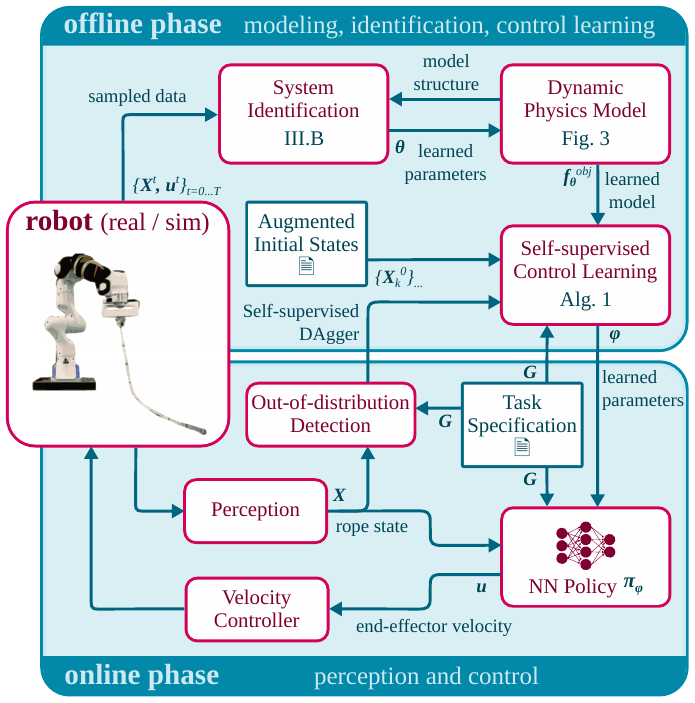}
    \caption{\textbf{The data flow of the SPiD framework}.}
    \label{fig:framework}
    \vspace{-0.2cm}
\end{figure}

In this section, we present the SPiD framework in three stages.
First, we describe how we model a DLO in Sec.~\ref{sec:modeling} and how its parameters are identified in Sec.~\ref{sec:identification}.
Then, in Sec.~\ref{sec:nn-control}, we present the augmented self-supervised control learning algorithm together with the self-supervised DAgger mechanism.

\subsection{Modeling} \label{sec:modeling}

As shown in Fig.~\ref{fig:mass-spring-model}, a DLO is modeled as a series of mass points connected by spring–damper elements. Each component—mass point, spring, and damper—is assigned independent parameters, providing the model with sufficient representational capacity to capture a wide range of physical properties, including those of heterogeneous ropes.

\begin{figure}[t]
    \centering
    \includegraphics[width=0.81\linewidth]{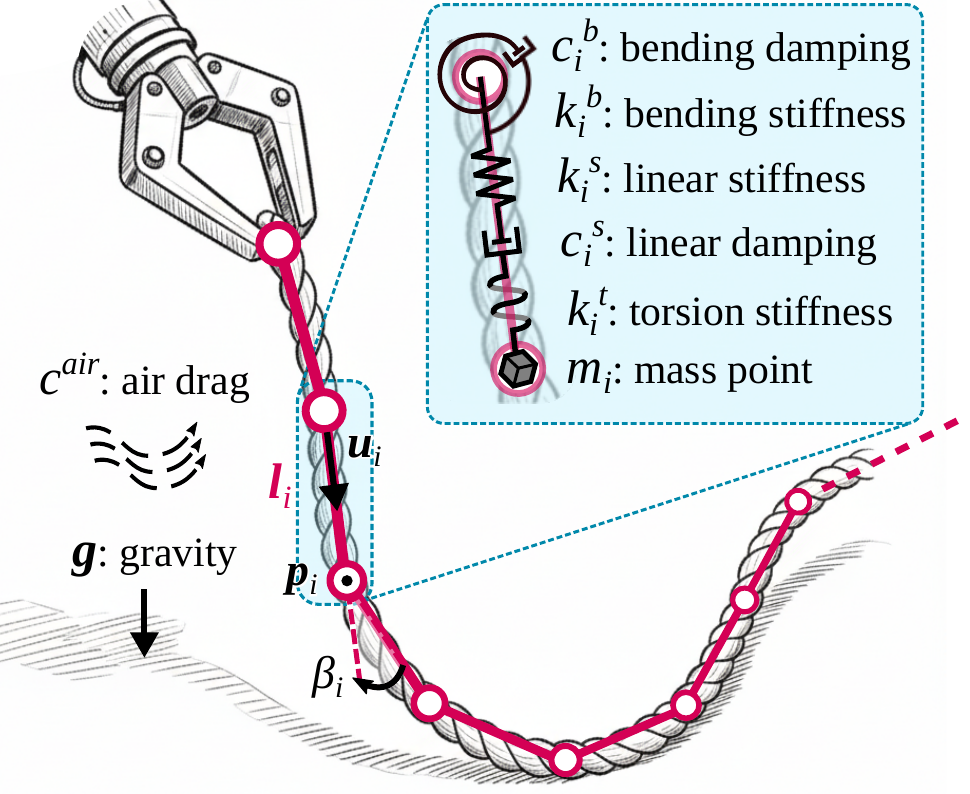}
    \vspace{-0.1cm}
    \caption{\textbf{The DLO model.} The model consists of $N{+}1$ mass points connected by linear springs/dampers, bending springs/dampers, and torsional springs. 
Additionally, air drag and gravity are incorporated into dynamics.}
    \label{fig:mass-spring-model}
    \vspace{-0.4cm}
\end{figure}

\textbf{Notation.}
A DLO consists of $N$ cable links $(1, 2, \dots, N)$ and $N\!+\!1$ mass points $(0, 1, \dots, N)$.  
Each point has a mass $m_i$, and its state is denoted as $\bm{x}_i = [\,\bm{p}_i^\mathsf{T},\, \bm{v}_i^\mathsf{T}\,]^\mathsf{T}$,  
where $\bm{p}_i \in \mathbb{R}^3$ and $\bm{v}_i \in \mathbb{R}^3$ represent the position and velocity vectors, respectively.  
Therefore, the full state of the DLO has a dimension of $6(N\!+\!1)$ and is represented as $\bm{X}=[\bm{x}^T_0,...,\bm{x}^T_{N}]^T$.  
The initial (rest) length of the $i$-th cable link is denoted as $l_i^0$.  
The current link vector is defined as $\bm{l}_i = \bm{p}_i - \bm{p}_{i-1}$, with its length $l_i = \lVert \bm{l}_i \rVert$ and unit direction $\bm{u}_i = \bm{l}_i / \lVert \bm{l}_i \rVert$. 
The bending angle between two consecutive links $\bm{l}_i$ and $\bm{l}_{i+1}$ is expressed as  
$\beta_i = \arccos\!\left( \bm{u}_i \cdot \bm{u}_{i+1} \right)$. Finally, $\psi_i$ denotes the torsion angle of the $i$-th torsion spring.

\textbf{Linear spring–damper force} resists the change of distance and relative velocity between two connected mass points. 
Its expression on mass-point $i$ is given by
\begin{equation}
\begin{split}
\bm{F}_i^{\,\text{lin}} &= - k^s_i \big( l_i - l_i^0 \big) \bm{u}_i
+ k^s_{i+1} \big( l_{i+1} - l_{i+1}^0 \big) \bm{u}_{i+1}
 \\
-& c^s_i \big( (\bm{v}_i - \bm{v}_{i-1})^\mathsf{T} \bm{u}_i \big) \bm{u}_i
+ c^s_{i+1} \big( (\bm{v}_{i+1} - \bm{v}_{i})^\mathsf{T} \bm{u}_{i+1} \big) \bm{u}_{i+1},
\label{eq:Linear spring–damper force}
\end{split}
\end{equation}
where $k_i^s$ and $c_i^s$ denote the stiffness and damping coefficients of the $i$-th linear spring–damper element, respectively.

\textbf{Bending spring force} is an elastic restoring force generated when the bending angle (or curvature) between two adjacent links deviates from its reference value. Its expression on mass-point $i$ is given by
\begin{equation}
\begin{split}
\bm{F}_i^{\,\text{b}} = &\frac{k^b_{i-1} \beta_{i-1}}{l_i}
\frac{\bm{u}_i \times (\bm{u}_{i-1} \times \bm{u}_i)}{\sin \beta_{i-1}}
\minus
\frac{k^b_{i} \beta_i}{l_{i+1}}
\frac{\bm{u}_{i+1} \times (\bm{u}_i \times \bm{u}_{i+1})}{\sin \beta_i} \\
& \hspace{-2.2em} -
\frac{k^b_{i} \beta_i}{l_i}
\frac{\bm{u}_i \times (\bm{u}_i \times \bm{u}_{i+1})}{\sin \beta_i}
\plus
\frac{k^b_{i+1} \beta_{i+1}}{l_{i+1}}
\frac{\bm{u}_{i+1} \times (\bm{u}_{i+1} \times \bm{u}_{i+2})}{\sin \beta_{i+1}},
\label{eq:Bending force}
\end{split}
\end{equation}
where $k_i^b$ denotes the bending stiffness of the $i$-th bending spring~\cite{lv2017physically}.
The bending spring force can be decomposed into two parts.
In Eq.~\eqref{eq:Bending force}, the first and fourth terms correspond to the forces exerted on mass point $i$ by its adjacent bending springs,
while the middle two terms represent the reaction forces generated by the $i$-th bending spring itself.

\textbf{Bending damping force} exerts viscous resistance to the rate of change of bending deformation.
During dynamic manipulation, where the deformation rate of the system is relatively high, the bending damping can have a non-negligible influence on the overall system dynamics.
In \cref{proposition1}, we derive the analytical formulation of the bending damping force.

\begin{proposition}[]Let $c_i^b$ denote the bending damping coefficient of the $i$-th bending damper. The bending damping force acting on the $i$-th mass point is then given by
\begin{equation}
\begin{split}
\bm{F}_i^{\,\text{db}} = &\frac{c^b_{i\minus 1} \dot{\beta}_{i\minus 1}}{l_i}
\frac{\bm{u}_{i\minus 1} \minus \bm{u}_i \cos \beta_{i\minus 1}}{\sin \beta_{i\minus 1}}
\plus
\frac{c^b_{i} \dot{\beta}_i}{l_{i\plus 1}}
\frac{\bm{u}_{i\plus 1} \cos \beta_i \minus \bm{u}_{i}}{\sin \beta_i} \\
 & \hspace{-1.5em}-
\frac{c^b_{i} \dot{\beta}_i}{l_i}
\frac{\bm{u}_i \cos\beta_i \minus \bm{u}_{i\plus 1}}{\sin \beta_i}
\plus
\frac{c^b_{i\plus 1} \dot{\beta}_{i\plus 1}}{l_{i\plus 1}}
\frac{\bm{u}_{i\plus 1} \cos\beta_{i\plus 1} \minus \bm{u}_{i\plus 2}}{\sin \beta_{i\plus 1}},
\label{eq:bending damping force}
\end{split}
\end{equation}
where $\dot{\beta}_i$ represents the rate of change of the bending angle, obtained as
\begin{equation}
\dot{\beta}_i
= 
\frac{\bm{v}_i \minus \bm{v}_{i\minus1}}{l_i}
\frac{\bm{u}_i \cos \beta_i \minus \bm{u}_{i\plus 1}}{\sin \beta_i}
\plus  
\frac{\bm{v}_{i\plus 1} \minus \bm{v}_i}{l_{i\plus 1}}
\frac{\bm{u}_{i\plus 1} \cos \beta_i \minus \bm{u}_{i}}{\sin \beta_i}.
\label{eq:bending_angle_rate}
\end{equation}
\label{proposition1}
\end{proposition}

\begin{proof}


The partial derivative of the bending angle $\beta_i$ with respect to the link vector $\bm{l}_i$ can be expressed as
\begin{equation}
\begin{aligned}
\frac{\partial \beta_i}{\partial \bm{l}_i}
&= \frac{\mathrm{d}(\arccos(\cos \beta_i))}{\mathrm{d}(\cos \beta_i)}
   \frac{\partial}{\partial \bm{l}_i}
   \!\left(
      \frac{\bm{l}_i \!\cdot\! \bm{l}_{i+1}}
           {\|\bm{l}_i\|\,\|\bm{l}_{i+1}\|}
   \right) \\
&= \minus\frac{1}{\sin \beta_i}
   \frac{\|\bm{l}_i\|^2\bm{l}_{i+1}\minus(\bm{l}_i \cdot \bm{l}_{i+1})\,\bm{l}_i}
   {\|\bm{l}_i\|^3\,\|\bm{l}_{i+1}\|}
= \frac{\bm{u}_i \cos\beta_i \minus \bm{u}_{i+1}}{\sin \beta_i\,\|\bm{l}_i\|}.
\label{pf:1}
\end{aligned}
\end{equation}

Similarly, the derivative of $\beta_i$ with respect to the adjacent link $\bm{l}_{i+1}$ is given by
\begin{equation}
\frac{\partial \beta_{i}}{\partial \bm{l}_{i+1}}
= \frac{\bm{u}_{i+1} \cos\beta_i - \bm{u}_{i}}{\sin \beta_i\,\|\bm{l}_{i+1}\|}.
\label{pf:2}
\end{equation}

By applying the chain rule, the derivatives of $\beta_i$ with respect to the three connected mass points
$\bm{p}_{i-1}$, $\bm{p}_i$, and $\bm{p}_{i+1}$ satisfy the following relations:
\begin{equation}
\frac{\partial \beta_i}{\partial \bm{p}_{i-1}}
\equals \minus\frac{\partial \beta_i}{\partial \bm{l}_i},
\quad
\frac{\partial \beta_i}{\partial \bm{p}_i}
\equals \frac{\partial \beta_i}{\partial \bm{l}_i}
 \minus \frac{\partial \beta_i}{\partial \bm{l}_{i+1}},
\quad
\frac{\partial \beta_i}{\partial \bm{p}_{i+1}}
\equals \frac{\partial \beta_i}{\partial \bm{l}_{i+1}}.
\label{pf:3}
\end{equation}

Consequently, the time derivative of the bending angle $\beta_i$ can be written in terms of the velocities of the three adjacent mass points:
\begin{equation}
\dot{\beta}_i
= 
\frac{\partial \beta_i}{\partial \bm{p}_{i-1}} \cdot \bm{v}_{i-1}
+ \frac{\partial \beta_i}{\partial \bm{p}_{i}} \cdot \bm{v}_{i}
+ \frac{\partial \beta_i}{\partial \bm{p}_{i+1}} \cdot \bm{v}_{i+1}.
\label{pf:4}
\end{equation}

Substituting \Cref{pf:1,pf:2,pf:3} into \cref{pf:4} and rearranging yields \cref{eq:bending_angle_rate}.  
According to the principle of virtual work, 
$\delta W = -c_i^b\,\dot{\beta}_i\,\delta \beta_i 
= \sum_k \bm{F}_{i \to k}^{\,\mathrm{db}} \cdot \delta \bm{p}_k$, 
where $\bm{F}_{i \to k}^{\,\mathrm{db}}$ denotes the force exerted by the $i$-th bending damper on $k$-th mass point, which can be expressed as
\begin{equation}
\bm{F}_{i \to k}^{\,\mathrm{db}}
= -c_i^b\,\dot{\beta}_i\,
  \frac{\partial \beta_i}{\partial \bm{p}_k},
\quad
k \in \{\,i\!-\!1,\, i,\, i\!+\!1\,\}.
\label{pf:5}
\end{equation}

Finally, the total bending damping force acting on the $i$-th mass point is obtained by summing the contributions from the three neighboring bending dampers $(i\!-\!1,i,i\!+\!1)$:
\[
\bm{F}_i^{\,\mathrm{db}}
= \bm{F}_{i-1 \to i}^{\,\mathrm{db}}
 - \bm{F}_{i \to i-1}^{\,\mathrm{db}}
 - \bm{F}_{i \to i+1}^{\,\mathrm{db}}
 + \bm{F}_{i+1 \to i}^{\,\mathrm{db}}.
\]
By substituting \cref{pf:1,pf:2,pf:3,pf:5} into the above expression and simplifying, we obtain \cref{eq:bending damping force}.
\end{proof}

\textbf{Torsion spring force} consists of two components: material torsion and geometric torsion. Material torsion originates from the internal shear stress induced by twisting along the rope’s longitudinal axis, representing the material’s elastic resistance to angular deformation~\cite{theetten2008geometrically,valentini2011modeling}.
In a rope manipulation task, the free end allows this torsional energy to dissipate quickly through rotation, rendering the material torsion negligible. Hence, the torsion angle $\psi_i$  is the geometric torsion angle in our model.

Each cable link $\bm{l}_i$ together with its adjacent links $\bm{l}_{i-1}$ and $\bm{l}_{i+1}$ forms two planes, and the angle between these planes defines the geometric torsion angle $\psi_i$. 
By taking the partial derivative of the torsional potential energy with respect to the position of mass point $i$, 
the torsion spring force acting on this point can be obtained~\cite{lv2017physically} as
\begin{equation}
\bm{F}_i^{\,t} 
= -\frac{\partial E^t}{\partial \bm{p}_i}, 
\quad 
E^t = \frac{1}{2} \sum_{j=2}^{N-1} k_j^t \psi_j^2,
\label{eq:Torsion spring force}
\end{equation}
where $k_j^t$ denotes the torsional stiffness of the $j$-th torsion spring. 

\textbf{External forces} in our model consist of gravity and air resistance. 
The total external force acting on the $i$-th mass point is expressed as
\begin{equation}
\begin{split}
\bm{F}_i^{\,\text{ext}} = m_i\bm{g} - c^{air}\bm{v}_i,
\label{eq:External forces}
\end{split}
\end{equation}
where $\bm{g}$ denotes the gravitational acceleration vector, and $c^{\text{air}}$ is the air damping coefficient.


\textbf{Forward time integration} refers to the numerical process of advancing the system states over time by discretizing the given equations of motion. By combining \cref{eq:Linear spring–damper force,eq:Bending force,eq:Torsion spring force,eq:bending damping force,eq:bending_angle_rate,eq:External forces} with Newton’s second law, the equations of motion for each mass point are obtained
\begin{equation}
\begin{aligned}
\bm{a}_i &= 
\left( \bm{F}_i^{\,\text{lin}}
+ \bm{F}_i^{\,\text{b}}
+ \bm{F}_i^{\,\text{t}}
+ \bm{F}_i^{\,\text{db}}
+ \bm{F}_i^{\,\text{ext}} \right) / m_i,
\label{eq:dynamic}
\end{aligned}
\end{equation}
which is then integrated forward in time using the \emph{symplectic Euler} scheme
\begin{equation}
\begin{aligned}
\bm{v}_i^{t+1} &= \bm{v}_i^{t} + \Delta t\,\bm{a}_i^{t},\\
\bm{p}_i^{t+1} &= \bm{p}_i^{t} + \Delta t\,\bm{v}_i^{t+1},
\label{eular}
\end{aligned}
\end{equation}
where $\Delta t$ denotes the integration time step, and the superscript $t$ indicates the current time step.

\textbf{Incorporating control input.}
So far, the DLO has been modeled as a discrete autonomous system.
To obtain a fully controlled DLO model, we now introduce the control input $\bm{u}$.
The top end of the DLO is attached to the end-effector of a manipulator, which in our experiments is driven by a Cartesian-space velocity controller. This low-level controller is capable of accurately tracking smooth and physically feasible velocity commands, and for lightweight objects like ropes, the reaction force generated by its swinging on the manipulator can be neglected. Therefore, it is reasonable to assume that the commanded end-effector velocity $\bm{u}$ directly determines the velocity of the first mass point of the DLO:
\begin{equation}
\bm{v}_0 = \bm{u} \ \text{(control input)}.
\label{v=u}
\end{equation}
Under this constraint, the acceleration of the first mass point $\bm{a}_0$ becomes redundant and does not participate in the system evolution. Combining condition \eqref{v=u} with \cref{eq:dynamic,eular}, we obtain the discrete controlled DLO dynamics as
\begin{equation}
\bm{X}^{t+1} = \bm{f}^{\text{dlo}}_{\theta, \Delta t}(\bm{X}^{t}, \bm{u}^{t}),
\label{rope model}
\end{equation}
where $\theta$ denotes the complete set of physical model parameters (e.g., $k_i^s$, $c_i^s$, $k_i^b$, etc) and $\bm{X}$ denotes full DLO state $\bm{X}=[\bm{x}^T_0,...,\bm{x}^T_{N}]^T$.

\subsection{System identification}\label{sec:identification}
After deriving the mathematical model of the controlled system, the model parameters $\theta$ are identified to obtain a high-fidelity dynamic model that accurately captures the real system dynamics.
Specifically, $\theta$ consists of the particle mass vector $\mathbf{m} = \{m_i\}_{i=0}^{N}$, the initial link-length vector $\boldsymbol{l}^0 = \{l_i^0\}_{i=1}^{N}$, the gravitational acceleration vector $\boldsymbol{g}$, the aerodynamic drag coefficient $c^{air}$, and the rope stiffness and damping parameters, including the linear spring and damping vectors $\boldsymbol{k}^s = \{k_i^s\}_{i=1}^{N}$ and $\boldsymbol{c}^s = \{c_i^s\}_{i=1}^{N}$, the bending spring and damping vectors $\boldsymbol{k}^b = \{k_i^b\}_{i=1}^{N-1}$ and $\boldsymbol{c}^b = \{c_i^b\}_{i=1}^{N-1}$, as well as the torsional stiffness vector $\boldsymbol{k}^t = \{k_i^t\}_{i=2}^{N-1}$, as illustrated in Fig.~\ref{fig:mass-spring-model}.
In compact form, we denote the full parameter set as
\[
\theta = \{\mathbf{m},\, \boldsymbol{l}^0,\, \boldsymbol{g},\, c^{air},\, \boldsymbol{k}^s,\, \boldsymbol{c}^s,\, \boldsymbol{k}^b,\, \boldsymbol{c}^b,\, \boldsymbol{k}^t\}.
\]
Among these parameters, $\mathbf{m}$ and $\boldsymbol{l}^0$ can be obtained via direct measurement, $\boldsymbol{g}$ is known from physics, and the remaining parameters are estimated through system identification.

Since the model in \cref{rope model} is differentiable, a differentiable physics-based system identification approach~\cite{long2026dissipative,yang2025differentiable} is adopted.
For convenience, we denote the positions of all rope nodes as $\bm{P} = [\bm{p}_0^T,\ldots,\bm{p}_N^T]^T$ and the velocities as $\bm{V} = [\bm{v}_0^T,\ldots,\bm{v}_N^T]^T$, which are contained in $\bm{X}$.
Given a dataset $\{(\bm{P}^t,\bm{u}^t)\}_{t=0}^{T}$, the model and the reference system are initialized with the same state and driven by identical control inputs.
The identification loss is defined as
\begin{equation}
\mathcal{L}_{id}(\theta)
= \frac{1}{H}\sum_{t=1}^{H}
\left\|\hat{\bm{P}}^t(\theta)-\bm{P}^t\right\|_2^2,
\label{eq:id_loss}
\end{equation}
where $H$ denotes the identification horizon and $\hat{\bm{P}}^t(\theta)$ represents the predicted rope positions obtained by recursively applying the differentiable dynamics model.
The gradients of the loss with respect to the unknown parameters are computed through automatic differentiation, and the parameters are optimized using Adam, while the known parameters remain fixed.

In the real system, the motion-capture perception system only provides position measurements rather than velocities.
Therefore, only the position part of the state trajectory is used for comparison.
The rope remains stationary until the initial release, implying zero initial velocity.
Together with the initial position provided by the motion-capture system, the complete initial state of the rope is therefore available, with $\hat{\bm{P}}^0=\bm{P}^0$ and $\hat{\bm{V}}^0=\bm{0}$.

In the simulation experiments, the rope consists of 21 mass points, resulting in nearly one hundred parameters to be identified.
When processing long time-series data with such a large number of parameters, the increased recurrence depth amplifies the parameter-to-loss nonlinearity, leading to a highly non-convex loss landscape and causing the optimization to be frequently trapped in suboptimal local minima.
To address this issue, two strategies are employed.
The first is a variable-horizon strategy, where a relatively short horizon $H$ is used at the beginning, and $H$ is gradually increased once the loss falls below a predefined threshold, until it eventually reaches the full data length $T$.
The second is a two-stage parameter strategy.
In the first identification stage, all springs and dampers of the same type are constrained to share a common parameter across the rope, so that the rope is treated as a homogeneous system.
After this stage converges, the identified parameters are used as initialization, and the constraint is relaxed so that each spring and damping parameter is treated independently, forming a heterogeneous rope model for a refined identification stage.
These two strategies improve optimization stability and help the model achieve better identification results.

Finally, since the sampling interval $h$ and the integration time step $\Delta t$ of the model may differ, temporal alignment between the predicted and measured trajectories is required.
Specifically, the dynamics model is integrated for $h/\Delta t$ steps between consecutive measurements.
In the real-world experiments, the motion-capture system provides measurements at a sampling interval of $10$\,ms, whereas the integration time step of our model is $1$\,ms.
Therefore, the model is propagated for 10 integration steps between consecutive measurements.

\subsection{Augmented self-supervised neural controller learning}
\label{sec:nn-control} \label{our controller}

After obtaining an accurate physical model, we are able to train a task-oriented controller based on it.
At this phase, we train a neural network to learn behaviors that minimize a given task-oriented loss function $\bm{L_\text{task}}(\cdot)$.
The controller $\pi_\phi(\cdot)$ takes as input the current rope state $\bm{X} = [\bm{x}_0^T,\ldots,\bm{x}_N^T]^T$ together with the task specification $\bm{G}$, and outputs a 3D velocity command $\bm{u}$ for the robot end-effector in Cartesian space.
The specific designs of $\bm{L_\text{task}}(\cdot)$ and $\bm{G}$ for the rope stabilization and trajectory tracking tasks will be presented in Sec.~\ref{Experiments}.

\begin{algorithm}[t]
\caption{Physics-informed Self-supervised Training}
\textbf{Input:} Initial state set $\{\bm{X}_k^0\}_{k=1,\dots,K}$, task specification set $\{\bm{G}_k\}_{k=1,\dots,K}$, model parameters $\theta$, integration time step $\Delta t$\\
\textbf{Output:} Trained controller parameters $\phi$
\begin{algorithmic}[1]
    \State Initialize controller parameters $\phi$
    \For{each training iteration}
        \State \label{line:randomization1} $\{\Delta \theta_k\}_{k=1,\dots,K} \gets \text{sample\_noise}(\theta)$
        \For{$t = 0, \dots, T-1$}
            \State $\{\bm{u}_k^t\}_{k=1,\dots,K} \gets \{\pi_{\phi}(\bm{X}_k^t, \bm{G}_k)\}_{k=1,\dots,K}$
            \State $\{ \bm{X}_k^{t+1} \}_{k=1,\dots,K}
            \gets \{ \bm{f}^{\text{dlo}}_{\theta + \Delta \theta_k,\Delta t}(\bm{X}_k^t, \bm{u}_k^t) \}_{k=1,\dots,K}$ \label{line:randomization2}
        \EndFor
        \State $\bm{L} \leftarrow \frac{1}{K} \sum_{k=1}^K \bm{L_\text{task}}\!\left(\{\bm{X}^t_k\}_{t=1}^{T}, \bm{G}_k \right)$
        \Statex \hfill $\triangleright$ Task-based $\bm{L}_\text{task}(\cdot)$ as in \eqref{eq:loss1} and \eqref{eq:loss_track}
        \State $\phi \leftarrow \mathrm{Adam}(\phi, \nabla_\phi \bm{L})$
        \Comment{Autodiff}
    \EndFor
\end{algorithmic}
\label{alg_ctl}
\end{algorithm}

The self-supervised learning process is summarized in \cref{alg_ctl}.
In each training iteration, a batch of $K$ rollouts is executed in parallel.
In each training rollout, the controller interacts with the rope model for $T$ simulation steps, producing a state trajectory $\{\bm{X}_k^t\}_{t=1}^{T}$.
This trajectory, together with the task specification $\bm{G}_k$, is then fed into the task loss function to compute the batch loss.
Finally, the gradients are computed via backpropagation and used to update the controller parameters.

\textbf{Training augmentation:}
To enhance robustness and generalization, we incorporate several effective training strategies, including initial state diversification, geometric data augmentation, and domain randomization.

First, to increase initial state diversity, we apply random motions to the rope and record its states at different time steps.
The data collected during system identification are also included in the initial state set $\{\bm{X}_k^0\}_{k=1,\ldots,K}$.
In total, around $10k$ states are collected using this approach.
However, these states still only cover a few sparse trajectories in the 3D physical space.
To further expand the state distribution, we apply geometric data augmentation.
All recorded rope states are first translated so that their top endpoints coincide at a common point.
Then, we uniformly rotate each rope around the $z$-axis by small angular increments, covering $360^\circ$.
Lastly, the replicated states are further translated along the $x$, $y$, and $z$ directions with equal spacing.
In total, approximately 2 million states are generated for training, providing much broader coverage of the 3D physical space.

In addition, domain randomization is applied to improve robustness to variations in rope properties.
During training, we sample $K$ Gaussian noise vectors $\Delta\theta_k$ in each iteration and add them to the original model parameters to simulate ropes with varying physical characteristics (Alg.~\ref{alg_ctl}: lines~\ref{line:randomization1} and~\ref{line:randomization2}).
Experimental results show that the controller trained with this strategy can effectively handle rope-property variations within the range considered in our simulation and real-world experiments.

\textbf{Self-supervised DAgger:}
Although the above training strategies increase the diversity of training data, the learned controller may still encounter unseen states during deployment, leading to performance degradation due to distribution shift.
To address this issue, we propose a self-supervised DAgger mechanism that extends the core idea of DAgger~\cite{ross2011reduction} to self-supervised learning, enabling automatic detection and correction of distribution shifts without requiring expert demonstrations.

The key idea is to identify poorly handled states through task-specific performance feedback and aggregate them for further self-supervised policy refinement.
Since the controller is trained to minimize a task-oriented loss, its performance can be evaluated by monitoring the corresponding task loss during deployment.
When the loss exhibits abnormal behavior, the associated states are identified as OOD samples.
These samples, together with their corresponding task specifications, are then added back into the initial-state and task-specification sets with higher sampling weights.
The controller is subsequently refined offline through the same differentiable physics-based self-supervised training process in \cref{alg_ctl}, without requiring expert action labels.

The OOD detection criterion depends on the specific task.
For rope stabilization, the rope energy is used as the task loss, and the controller is expected to continuously dissipate this energy until the rope comes to rest.
During deployment, if the energy abnormally increases or oscillates around a positive value, an OOD event is triggered and the corresponding state is recorded.
For rope trajectory tracking, the task loss is defined based on the deviation between the rope-tip trajectory and the target trajectory.
During deployment, when the real-time tracking error exceeds a predefined threshold, an OOD event is triggered, and the corresponding state and trajectory segment are recorded.

By combining task-based OOD detection, data aggregation, and self-supervised policy refinement, the proposed mechanism enables the controller to improve its performance on previously unseen states without relying on expert demonstrations.

\section{Rope Perception}\label{sec:perception}

Estimating the rope states constitutes a challenging perception problem as ropes are thin and lightweight objects that are hard to distinguish visually, especially when they move fast.
High frequency and high accuracy perception is vital for achieving agile manipulation. 
Thus, we primarily used high-performance marker-based perception to evaluate the control performance of SPiD and the baseline. 
In addition, to assess practical deployment under a low-cost setup, we also consider a markerless single-camera pipeline.
Despite the physical limitations of the markerless system, our experiments show successful rope stabilization in Sec.~\ref{sec:exp-markerless}. 
We detail both perception approaches below. 

\subsection{Marker-based perception} 
We use off-the-shelf OptiTrack motion capture system as our default perception method. 
It requires multiple infrared cameras and reflective markers on the objects to track and reconstruct the 3D positions of the markers.
For rope tracking, a series of equally distanced markers are attached on the rope. 
The resulting measurements provide high-precision and high-rate state feedback as an external perception module.
Specifics of our setup are presented in Sec.~\ref{sec:real-world-stabilization}.

\subsection{Markerless perception}\label{sec:markerless-method} 

\highlight{For a more practical and affordable alternative, we develop a markerless vision pipeline based on a single RGB camera without external motion capture systems or artificial markers.} 
\highlight{In the present work, this setup is used as a secondary practical extension and provides planar observations in our experiments.}
\highlight{The pipeline consists of two stages: rope segmentation from RGB images, followed by splitting the segmented rope into equidistant points to recover the state representation used by the controller.}

\subsubsection{Rope detection and segmentation} 
\highlight{We evaluated several segmentation pipelines, including Grounding DINO + SAM, Grounding DINO + Mobile-SAM, and a fine-tuned YOLOv11-seg model \cite{liu2024grounding,Kirillov_2023_ICCV,zhang2023faster,ultralyticsyolo}. We selected YOLOv11-seg because it provided the best trade-off between accuracy and runtime in our experiments (Sec.~\ref{sec:exp-markerless}).}

\highlight{To adapt the model to rope perception, we collected a dedicated RGB dataset with instance-level rope mask annotations and fine-tuned YOLOv11-seg for joint detection and segmentation. During inference, the model outputs rope masks and confidence scores in a single forward pass. Low-confidence detections are rejected, and valid mask regions are merged into a single rope segmentation. Dataset details and quantitative comparisons are reported in Sec.~\ref{sec:exp-markerless-segmentation}.}

\subsubsection{Rope splitter algorithm} 
\highlight{The controller requires the rope state to be represented by equidistant points along the rope, consistent with the model in Sec~\ref{sec:modeling}.}
\highlight{Given the segmentation mask, we first apply \texttt{skeletonize}\footnote{scikit-image.org} to extract its centerline points $\bm{P}_\text{m} = \{\bm{p}_{m,j} \in \mathbb{R}^2| j=0\ldots M\}$. We then iteratively select split points $\bm{p}_{s,i}$ along this centerline, starting from $\bm{p}_{s,0}=\bm{p}_{m,0}$. For each current split point, the next one is chosen as the nearest point whose distance is at least a predefined splitting distance $d_{\min}$ and that does not move backward along the rope:}

\begin{equation}
\begin{aligned}
    &\bm{p}_{s,i+1} \leftarrow
\arg\min_{\bm{p}_{m,j} \in \bm{P}_m}
\;
d_{i,j}
\quad \text{s.t.}\quad\\
    d_{i,j}& \ge d_{\min},\quad\;
    d_{i,j} < \|\bm{p}_{m,j}-\bm{p}_{s,i{-}1}\|.
\end{aligned}
\end{equation}


\highlight{To reduce the nearest-neighbor search cost, we organize the centerline points in a KD-tree at each frame. In our preliminary tests, this reduced the splitting time from 1092 ms to 15 ms for a mask with 3300 pixels. Since this procedure approximates the rope by piecewise straight segments, high curvature can introduce drift. We reduce this effect by first splitting at a higher resolution (e.g., $cN$ points) and then selecting $N$ equally spaced final points.}

\section{Experiments and Results}\label{Experiments}

We evaluate our contributions in four sections. 
\mbox{Section~\ref{sec:exp-model-val}} presents the model validation experiment that measures the effect of the novel modeling terms. 
Section~\ref{sec:exp-stabilization} consists of our core experiments, showing the performance of the proposed controller in an agile rope manipulation task, under changing object configurations and external perturbations.
In Section~\ref{sec:exp-markerless}, the rope perception approach is replaced with the more affordable but low-fidelity markerless perception. SPiD achieves effective rope stabilization also in this condition. 
Section~\ref{sec:exp-trajectory} demonstrates the generalization capacity of SPiD by applying it on the rope trajectory tracking task, which is distinct from rope stabilization.

\subsection{Model validation \label{sec:exp-model-val}}

\begin{figure*}[t]
    \centering
    \includegraphics[width=0.95\linewidth]{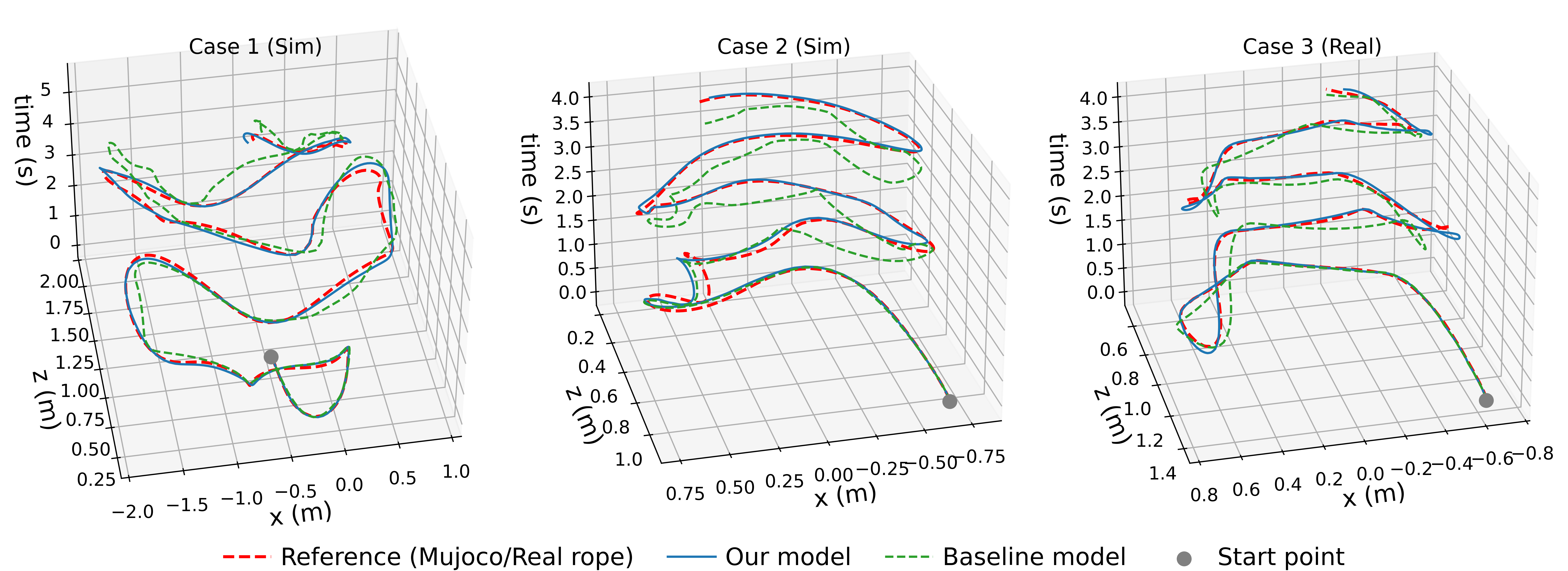}
    \caption{\textbf{Model validation. }Rope tip trajectories of the reference rope (Mujoco or real), \textit{our model}, and the \textit{baseline model} under three unseen initial states. Our model can predict future states with higher accuracy than the baseline model.
    }
    \label{fig:model validation}
\end{figure*}

\begin{figure*}[t]
    \centering
    \includegraphics[width=1.0\linewidth]{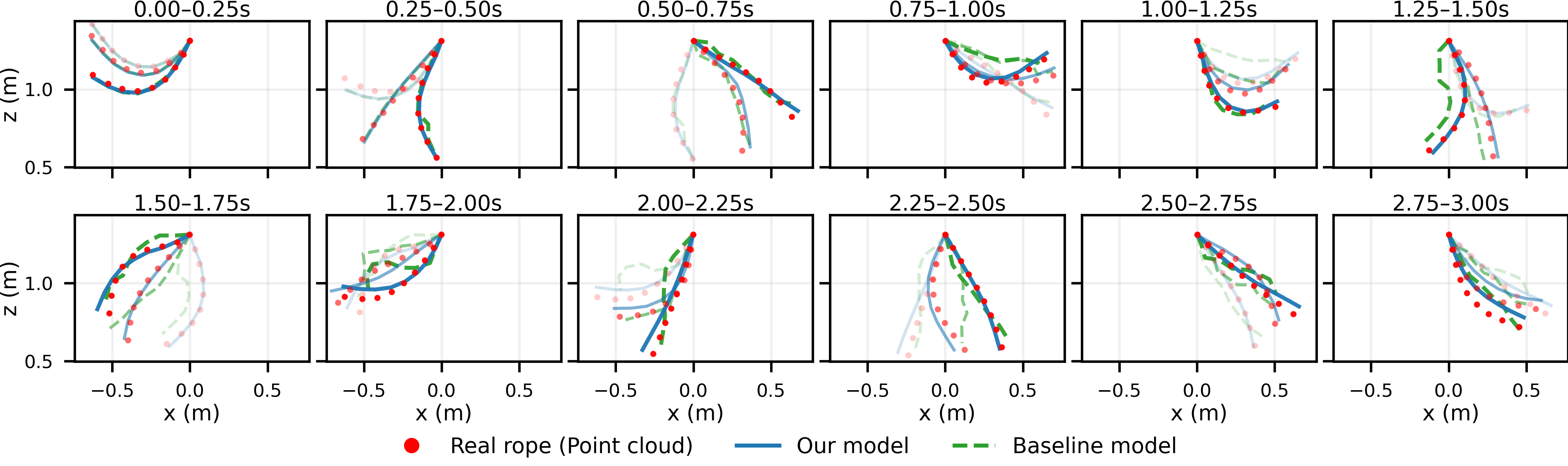}
    \caption{\textbf{Rope's motion sequence of Case 3 (Real).} The motion trajectory of the rope over the first 3\,s after release. Each subplot corresponds to a 0.25\,s time window and shows three uniformly sampled rope configurations. Lighter colors indicate earlier time instances.
    }
    \label{fig:motion sequence}
    \vspace{-0.4cm}
\end{figure*}

Compared to the previous mass–spring model \cite{lv2017physically}, we incorporate damping components to better capture the dynamic characteristics of the rope. To verify the higher fidelity of \textit{our model}, the previous method is used as the \textit{baseline model} \cite{lv2017physically} for comparison.
We validate the proposed approach both in simulation using the MuJoCo platform and in real-world experiments.
\highlight{Accordingly, two rope models are identified using the
method described in Sec.~\ref{sec:identification}: one for simulation and one for the real world. The simulated rope is 1 m long and represented by 21 nodes, while the real-world rope is 0.8 m long and represented by 9 nodes.}
\highlight{For parameter identification, both the \textit{baseline model} and \textit{our model} are identified using identical free-swing trajectory datasets, with 5,000 rope-state samples in simulation and 2,000 samples in the real-world experiment. In both cases, the top grasping point is fixed, and the rope is displaced by lifting its free tip. After release, the free-swing motion is recorded from MuJoCo in simulation and from the motion-capture system in the real-world experiment, and is then used for parameter optimization.}
During evaluation, two unseen initial states are tested in the MuJoCo simulation, and one unseen initial state is tested in the real-world experiment.
The RMSE between the rope tip position and the reference rope tip position is used as a quantitative metric to assess model accuracy.



The identification results are shown in \cref{fig:model validation}, where each curve represents the trajectory of the corresponding rope tip. 
\highlight{The reference rope and the two identified dynamics models} start from identical unseen initial states and then evolve solely according to their own system dynamics.
For clarity of visualization, the rope motion is constrained to mainly lie in the \(x\text{--}z\) plane, while the displacement along the \(y\)-axis is relatively small and therefore neglected.
As a result, \cref{fig:model validation} illustrates only the rope tip \(x\text{--}z\) trajectories over time.

In the MuJoCo simulation experiments, the first case adopts a relatively small initial bending angle, and the top end of the rope is moved within the \(x\text{--}z\) plane over a total duration of \(5.5\,\mathrm{s}\).
In contrast, the second case uses a larger initial bending angle, while the top end of the rope is fixed for \(4\,\mathrm{s}\).
As shown in Case~1 and Case~2 of \cref{fig:model validation}, our model demonstrates more accurate modeling performance.
Specifically, our model closely tracks the ground truth throughout the entire experiment, whereas the baseline model matches the ground truth for only approximately \(3.5\,\mathrm{s}\) in the first case and approximately \(1.5\,\mathrm{s}\) in the second case.
Similar trends are reflected in \cref{tab:rmse1}, where our model achieves significantly lower RMSE values compared to the baseline model.
We also measured the computational cost of the simulated rope model: the second-stage heterogeneous identification took $127.81 \pm 2.99$ min over three trials on an RTX 3060 GPU, while one single-step dynamics prediction took $0.345 \pm 0.045$ ms over 100 trials on an Intel i5-12500H CPU.

\begin{table}[t]
\centering
\caption{RMSE of the rope tip position prediction \highlight{(m)} in three experiment cases.}
\label{tab:rmse1}
\begin{tabular}{lccc}
\toprule
Model $\backslash$ Case & Case 1 (Sim) & Case 2 (Sim) & Case 3 (Real) \\
\midrule
Baseline model & 0.1510 & 0.1347 & 0.0810 \\
Our model    & \textbf{0.0704} & \textbf{0.0442} & \textbf{0.0591} \\
\bottomrule
\end{tabular}
\vspace{-0.2cm}
\end{table}

In the real-world experiment, a large initial bending angle is applied to the rope. After release, the top end of the rope is kept fixed, allowing the rope to swing freely. Case~3 in \cref{fig:model validation} shows the rope tip trajectories over the first \(4\,\mathrm{s}\).
To better illustrate the modeling performance, \cref{fig:motion sequence} presents the complete rope motion over the first \(3\,\mathrm{s}\) in Case~3. Red points denote the real rope node point cloud obtained from the OptiTrack system, the blue curve represents the rope motion predicted by our model, and the green dashed curve represents that predicted by the baseline model.
As can be observed, during the initial \(0.75\,\mathrm{s}\), both our model and the baseline model accurately reproduce the real rope motion. However, in the \(0.75\text{--}1.0\,\mathrm{s}\) interval, the baseline model begins to exhibit noticeable errors in the rope configuration. 
After this point, although it remains consistent with the real rope in terms of the overall swinging trend, significant distortions appear in the detailed rope shape.
In contrast, our model is able to accurately capture the evolution of the rope shape throughout the entire process. This observation is further supported by \cref{tab:rmse1}, which shows that our model achieves higher accuracy than the baseline model in Case~3.

\subsection{Rope stabilization}\label{sec:exp-stabilization}
We use the rope stabilization task as our primary testbed for evaluating our neural controller's agile manipulation performance. 
This task requires fast responses to the highly dynamic and nonlinear rope motions.

We demonstrate the performance of the proposed controller through simulation and real-world experiments.
In each of simulation and real-world settings, a single neural controller is trained. Then, it is tested across different rope types and initial positions to demonstrate its effectiveness and generalization capability.

In this task, we do not use task specification ($\bm{G}=\emptyset$) and use an energy-based loss for self-supervised control learning as detailed below.  

\subsubsection{Task formulation}\hfill\break\noindent
\indent\textbf{Loss function.}
We use the total energy of the system to quantify the degree of rope swing.
Accordingly, the controller is trained to perform maximum negative work at each step, thereby stabilizing the rope as soon as possible.
The rope energy consists of the gravitational potential and kinetic energies of all mass points, as well as the elastic potential energy stored in the springs:
{
\begin{equation}
\begin{split}
E_{\text{rope}}\left(\bm{X}\right)  =&
\sum\nolimits_{i=0}^{N}\!\left(m_i g z_i + \tfrac{1}{2} m_i \lVert \bm{v}_i \rVert^2\right)
+ \tfrac{1}{2}\!\sum\nolimits_{i=1}^{N-1}\! k_i^b \beta_i^2 \\
&+ \tfrac{1}{2}\!\sum\nolimits_{i=1}^{N}\! k_i^s (l_i - l_i^0)^2
+ \tfrac{1}{2}\!\sum\nolimits_{i=2}^{N-1}\! k_i^t \psi_i^2,
\label{eq:total_energy}
\end{split}
\end{equation}
}where $\bm{X}$ denotes the full state of the rope, and $z_i$ represents the third (vertical) component of the position vector $\bm{p}_i$.
The rope energy is defined relative to its static hanging equilibrium, where the total energy is set to zero.
In this task, the vertical component of the velocity command $\bm{u}$ is fixed to zero, as the loss function $E_{\text{rope}}(\cdot)$ accounts for gravitational potential energy; otherwise, the controller would continuously move the end-effector downward to minimize total energy.
\highlight{Thus, the stabilization controller uses planar end-effector actuation, while the rope state itself remains three-dimensional and can evolve in 3D space.}

During training, we only consider the rope energy at the final time step. Therefore, given a state trajectory $\{\bm{X}^t\}_{t=1}^{T}$, the training loss is defined as the rope energy at the terminal state:
\begin{equation}
\bm{L_\text{task}}\left(\{\bm{X}^t\}_{t=1}^{T}, \bm{G}= \emptyset\right) = E_{\text{rope}}(\bm{X}^{T}).
\label{eq:loss1}
\end{equation}

\begin{figure}[t]
    \centering
    \subfloat[\label{fig:a}]{\includegraphics[width=0.33\linewidth]{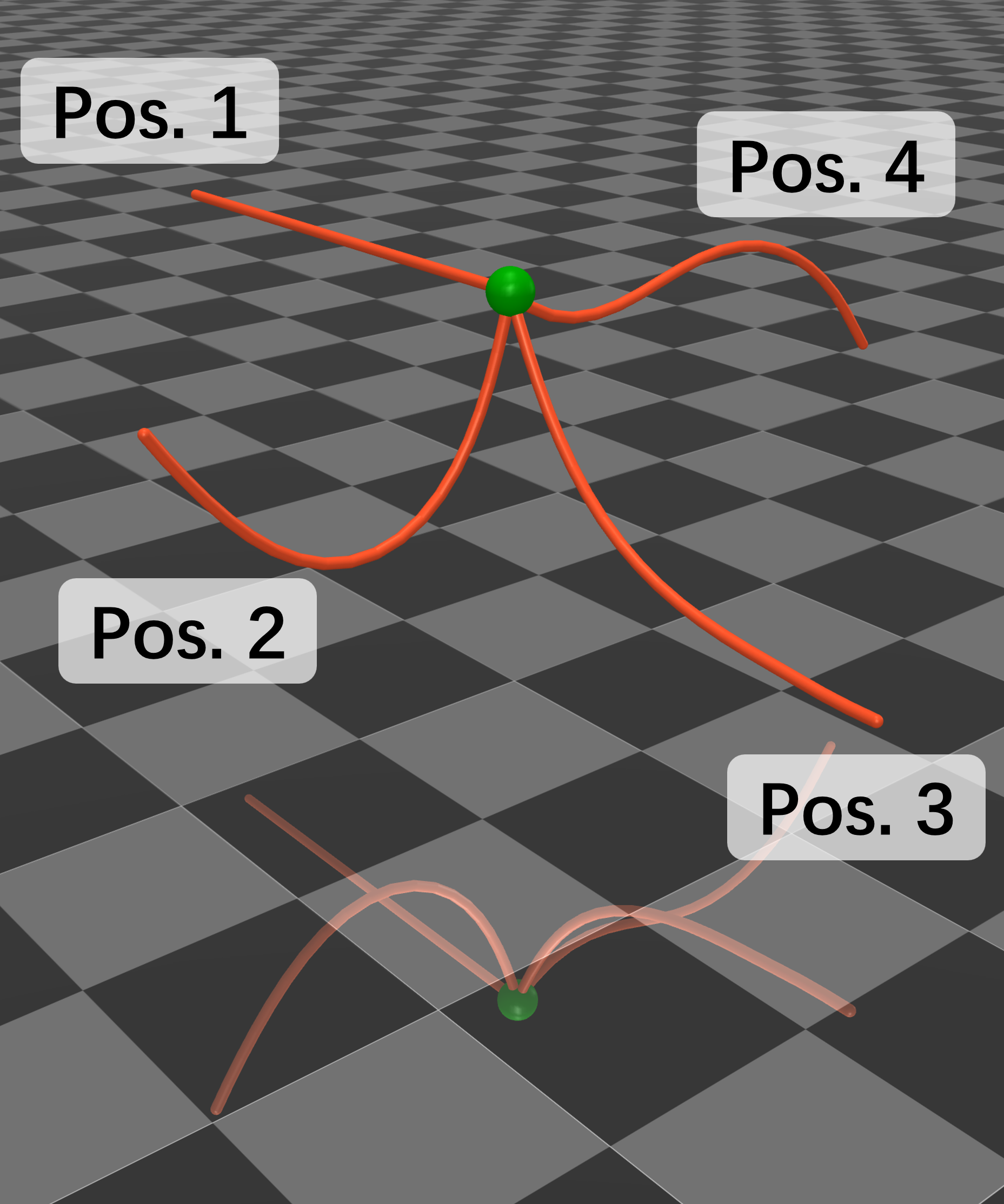}}
    \hfill
    \subfloat[\label{fig:b}]{\includegraphics[width=0.33\linewidth]{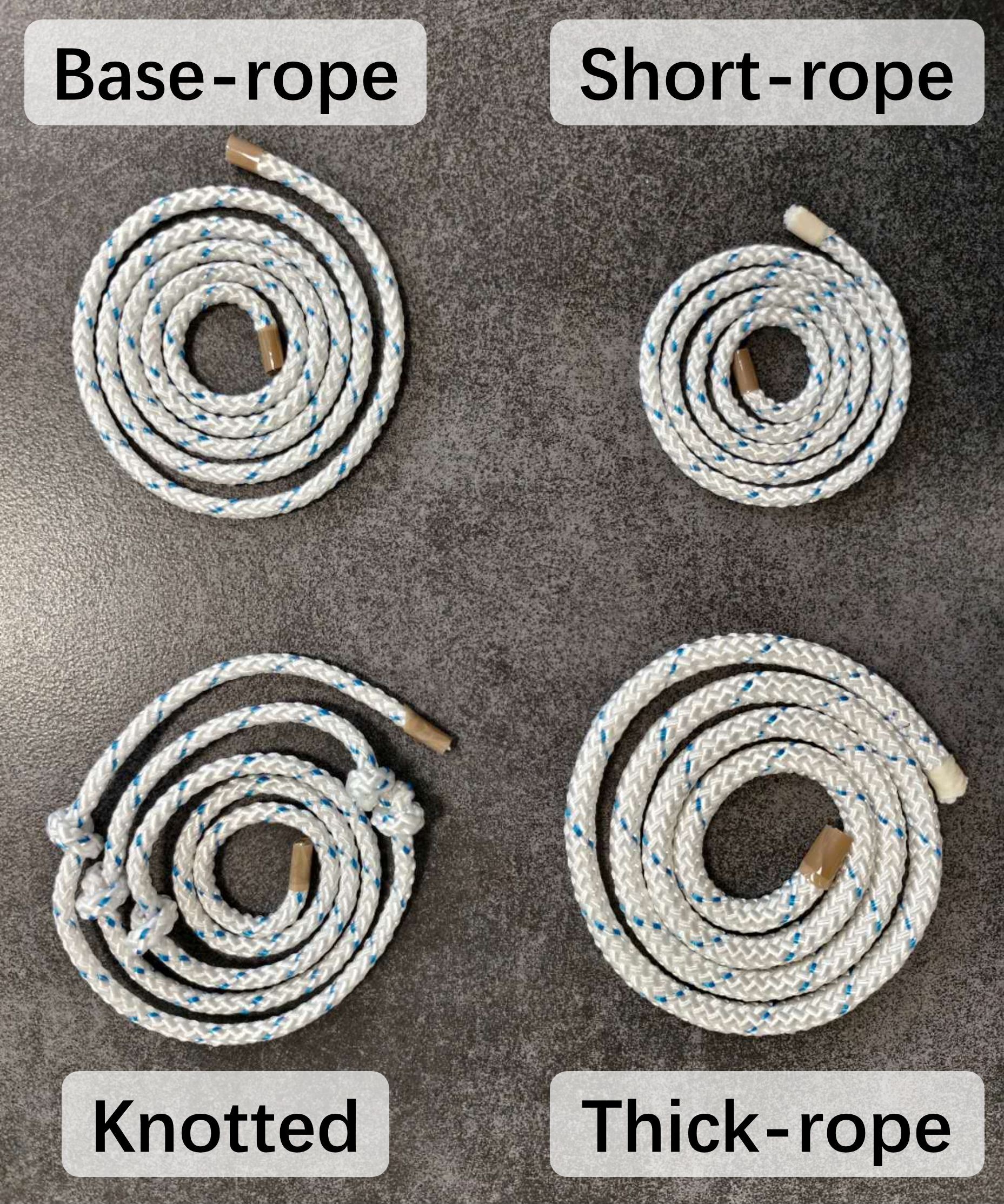}}
    \hfill
    \subfloat[\label{fig:c}]{\includegraphics[width=0.33\linewidth]{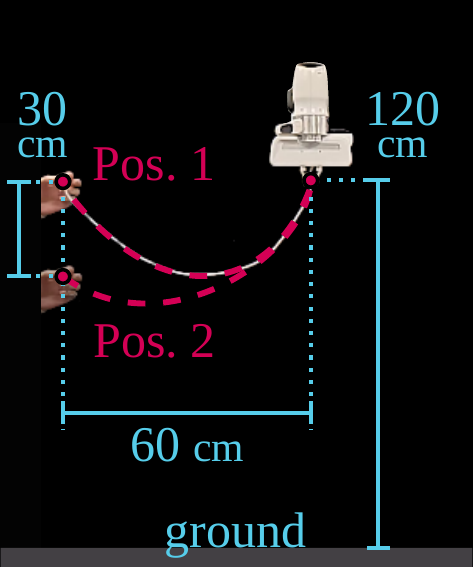}}
    \vspace{-0.1cm}
    \caption{\textbf{Test cases.}
    (a) Four unseen initial rope positions are used in simulation.
    (b) Four ropes with different masses, lengths and non-uniform mass distributions are used in the real world.
    (c) Two initial rope positions are used in the real world.}
    \label{fig:comparison}
    \vspace{-0.4cm}
\end{figure}

\textbf{Metrics.}
We continue to use the energy $E_{\text{rope}}(\cdot)$ defined in \cref{eq:total_energy} to measure the swinging motion of the rope in the evaluation.
The rope is considered stabilized when its energy decreases to $1\%$ of the initial value.

\begin{figure*}[h]
    \centering
    \includegraphics[width=\textwidth]{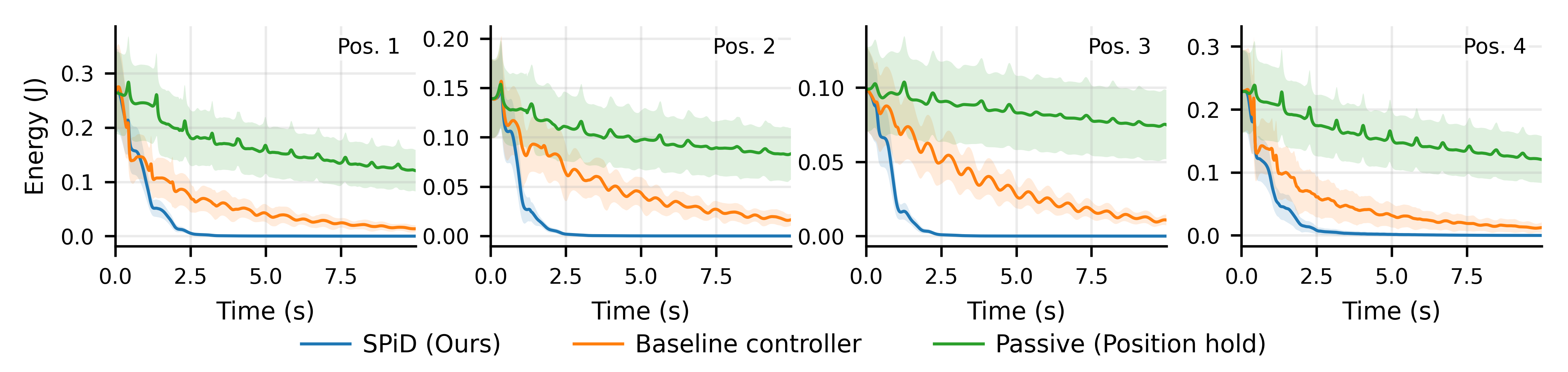}
    \vspace{-0.6cm}
    \caption{\textbf{Simulation results.} \highlight{Performance comparison among SPiD, the baseline controller, and the passive condition under four different initial rope positions.} 
    Each curve represents the mean energy profile over $5^3$ different rope settings. Shaded regions show one standard deviation. 
    SPiD achieves the fastest and most stable convergence, while the baseline controller exhibits oscillations in some cases.
}
    \label{fig:sim result}
    \vspace{-0.4cm}
\end{figure*}

\begin{figure*}[h]
    \centering
    \includegraphics[width=\textwidth]{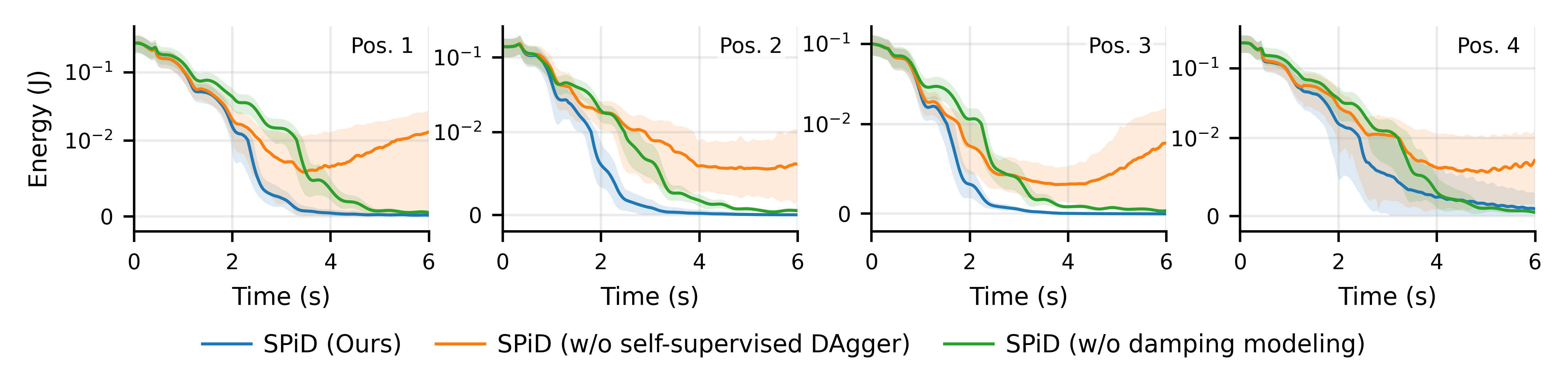}
    \vspace{-0.6cm}
    \caption{\highlight{\textbf{Ablation results.} Two ablation variants are compared with the full SPiD controller under the same simulation setup as in Fig. 7. “SPiD (w/o self-supervised DAgger)” removes the OOD-state collection and offline refinement stage, while “SPiD (w/o damping modeling)” trains the controller using the baseline dynamics model without the proposed damping terms. Curves show the mean rope energy over 5³ rope parameter settings, and shaded regions indicate one standard deviation. The full SPiD controller achieves faster and more stable energy dissipation across the four initial positions.}
}
    \label{fig:abl result}
    \vspace{-0.4cm}
\end{figure*}

\begin{figure*}[t]
    \centering
    \includegraphics[width=\textwidth]{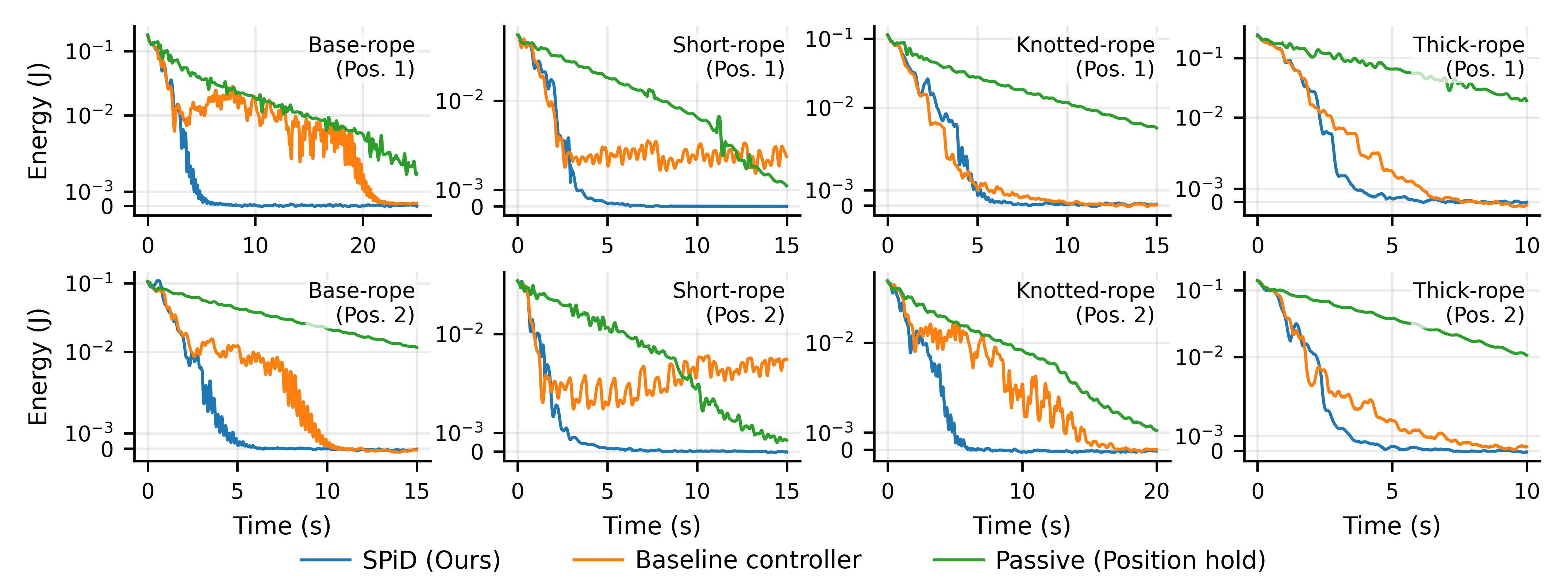}
    \vspace{-0.6cm}
    \caption{\textbf{Real-world results.} \highlight{Comparative experiments among SPiD, the baseline controller, and the passive condition were conducted on four different ropes under two initial positions.} Our controller consistently achieved the best performance across all scenarios, demonstrating both the effectiveness and generalization capability of the proposed algorithm.}

    \label{fig:real result}
    \vspace{-0.4cm}
\end{figure*}

        
        

\subsubsection{Baseline controller}
\label{baseline}
To validate the effectiveness of our proposed controller, we implement the baseline from~\cite{chen2017swing}. 
This baseline was originally developed for anti-swing control in a crane system, 
where the load is stabilized by controlling the motion of the cable's top end. 
The method models the cable as a double pendulum and proves the stability of the corresponding control law. 
In the original work, the control was applied only along the $x$-axis, 
while in our implementation, it has been extended to the $x$--$y$ plane.

\subsubsection{Simulation experiments}\hfill\break\noindent
\indent\textbf{Setup.}
The simulation experiments are conducted in MuJoCo 3.2.2 using a 1m rope model composed of 40 serially connected rigid segments. 
We manipulate the rope by controlling the velocity of its top end in the $x$--$y$ plane at $100\,\mathrm{Hz}$\highlight{, while the rope dynamics are still simulated with a 3D state representation}. 
Through the Python interface, the position and velocity of each segment's base can be accessed. 
We evaluate the generalization capability of the proposed algorithm on four different initial rope positions (\cref{fig:a}). For each initial position, we test the methods on $5^3$ rope \highlight{parameter settings obtained from a uniform grid over} the twist modulus in $[0,\, 1{\times}10^7]~\mathrm{Pa}$, bending modulus in $[0,\, 1{\times}10^7]~\mathrm{Pa}$, and segment mass in $[1.5,\, 3.5]~\mathrm{g}$.

\textbf{Results. \label{VB3}}
The simulation results are shown in \cref{fig:sim result}. 
The four subfigures correspond to the four initial positions illustrated in \cref{fig:a}. 
In each subfigure, the blue, orange, and green curves represent the energy evolution of the rope under our controller (Sec.~\ref{our controller}), the baseline method (Sec.~\ref{baseline}), and the passive condition, respectively. 
The passive case means that the top end of the rope remains fixed, and the energy decays naturally through air drag and internal damping. 
Each curve shows the average energy trajectory over $5^3$ different rope parameter settings. 
It can be observed that our controller achieves the fastest and smoothest performance in all cases, while the baseline controller takes longer and exhibits noticeable oscillations. 
\highlight{We additionally measured the runtime in this task: SPiD training using only the initial dataset took $28.69 \pm 2.85$ min over three trials on an RTX 3060 GPU, while online inference took $0.146 \pm 0.022$ ms over 100 trials on an Intel i5-12500H CPU.}

\highlight{
\textbf{Ablation study.}
We conduct two ablation studies under the same simulation setup as \cref{fig:sim result}, using the same four unseen initial rope positions and 5³ rope parameter settings. The results are shown in \cref{fig:abl result}.

To evaluate the role of model accuracy, we remove the proposed damping terms and train the controller with the baseline dynamics model, while keeping the same self-supervised learning and domain-randomization settings. This variant, denoted as SPiD (w/o damping modeling), dissipates rope energy more slowly than the full SPiD controller across the four initial positions. This shows that domain randomization cannot fully compensate for missing physical effects in the nominal dynamics model.

To evaluate the role of self-supervised DAgger, we remove the OOD-state collection and offline refinement stage, while keeping the remaining training settings unchanged. This variant, denoted as SPiD (w/o self-supervised DAgger), initially reduces the rope energy but tends to increase again after reaching a low-energy regime. This suggests that the initial training set does not sufficiently cover low-amplitude states, and that the DAgger refinement helps correct this distribution shift.

These results demonstrate that both the damping-enhanced dynamics model and the self-supervised DAgger mechanism contribute to faster and more stable rope stabilization.
}

\begin{table}[t]
    \centering
    \caption{Rope stabilization time (s) in real-world experiments. ``--'' indicates that the rope did not reach the stabilization condition within the experiment duration.}
    \label{tab:settle time}
    \setlength{\tabcolsep}{4pt}
    \begin{tabular}{llccc}
        \toprule
        & Rope type & Passive & Baseline  & SPiD (Ours) \\
        \midrule
        
        \multirow{4}{*}{\rotatebox{90}{Pos.1}}
            & Base-rope    & --   & 19.50 & \textbf{4.07} \\
            & Short-rope   & --   & --    & \textbf{3.69} \\
            & Knotted-rope & --   & \textbf{4.89}  & \textbf{4.77} \\
            & Thick-rope   & --   & 4.97  & \textbf{2.92} \\
        \midrule
        \multirow{4}{*}{\rotatebox{90}{Pos.2}}
            & Base-rope    & --   & 9.20  & \textbf{4.23} \\
            & Short-rope   & --   & --    & \textbf{3.58} \\
            & Knotted-rope & --   & 14.05 & \textbf{5.03} \\
            & Thick-rope   & --   & 5.29  & \textbf{3.24} \\
        
        \bottomrule
    \end{tabular}
\end{table}

\subsubsection{Real-world experiments}\hfill\break\noindent \label{sec:real-world-stabilization}
\indent\textbf{Setup.}
In real-world experiments, we use a Franka Emika Panda robotic arm equipped with a Franka Hand gripper, which grasps the top end of the rope. 
The desired velocity of the robot end-effector is updated at $100\,\mathrm{Hz}$, while the velocity controller runs at $1000\,\mathrm{Hz}$.
Rope segment positions are tracked at 100 Hz using the OptiTrack motion capture system consisting of 12 infrared cameras. For this purpose, we attached $9$ equally spaced reflective markers along the rope. 
The velocity of each marker is then estimated by differentiating the position data, followed by filtering to reduce noise, thereby reconstructing the full state of the rope \highlight{as a set of 3D point states}.

To evaluate the generalization capability of the proposed algorithm, four types of ropes are used, as shown in \cref{fig:b}:  
\textit{Base-rope} -- a cotton rope ($16\,\mathrm{g/m}$, $80\,\mathrm{cm}$ in length), also used for system identification;  
\textit{Short-rope} -- a shorter version of the base-rope ($16\,\mathrm{g/m}$, $60\,\mathrm{cm}$ in length);  
\textit{Knotted-rope} -- the base-rope tied with five knots ($80\,\mathrm{cm}$ after tying);  
\textit{Thick-rope} -- a thicker rope ($40\,\mathrm{g/m}$, $80\,\mathrm{cm}$ in length).  
To enrich the dynamic characteristics of the ropes, additional weights of $15\,\mathrm{g}$, $7\,\mathrm{g}$, $7\,\mathrm{g}$, and $15\,\mathrm{g}$ are attached to the tips of the four ropes, respectively.  
For each rope, two different initial positions are selected for testing, as illustrated in \cref{fig:c}.

\indent\textbf{Results.}
The real-world experiment results are shown in \cref{fig:real result}, for four types of ropes (\cref{fig:b}) and two initial positions (\cref{fig:c}). 
For the base-rope, our controller quickly stabilizes the rope, while the baseline controller experiences oscillations before eventually settling down, although it still performs significantly better than the passive case. 
This oscillation behavior arises because the baseline controller is derived from a simplified rope model, where the rope is approximated as a double pendulum. 
In reality, however, the rope exhibits more complex dynamics, and this model mismatch leads to degraded performance. 
In contrast, our controller explicitly accounts for the rope’s dynamic characteristics, resulting in faster and more stable control. 

When using a shorter rope, our method remains stable, whereas the baseline starts to oscillate after a certain point — even performing worse than the passive case.
This occurs because a shorter rope produces a larger swinging angle for the same top-end displacement, and such "overshooting" prevents convergence. 
Consequently, the baseline optimized for long ropes fails to generalize to shorter ones, while our method remains effective across different rope lengths.

For the knotted rope, the baseline achieves performance similar to our controller in the \textit{pos.~1} case, but it again exhibits oscillations in the \textit{pos.~2} case and takes longer to reach stability compared to the base-rope experiment.
Finally, for the thick rope, our controller still stabilizes the rope faster than both the baseline and the passive cases.

Our controller consistently achieves rope stabilization in the shortest time across all scenarios as reported in \cref{tab:settle time}.
Overall, our method demonstrates superior control performance and generalization capability across all scenarios.

\begin{figure}[t]
    \centering
    \includegraphics[width=1.0\linewidth]{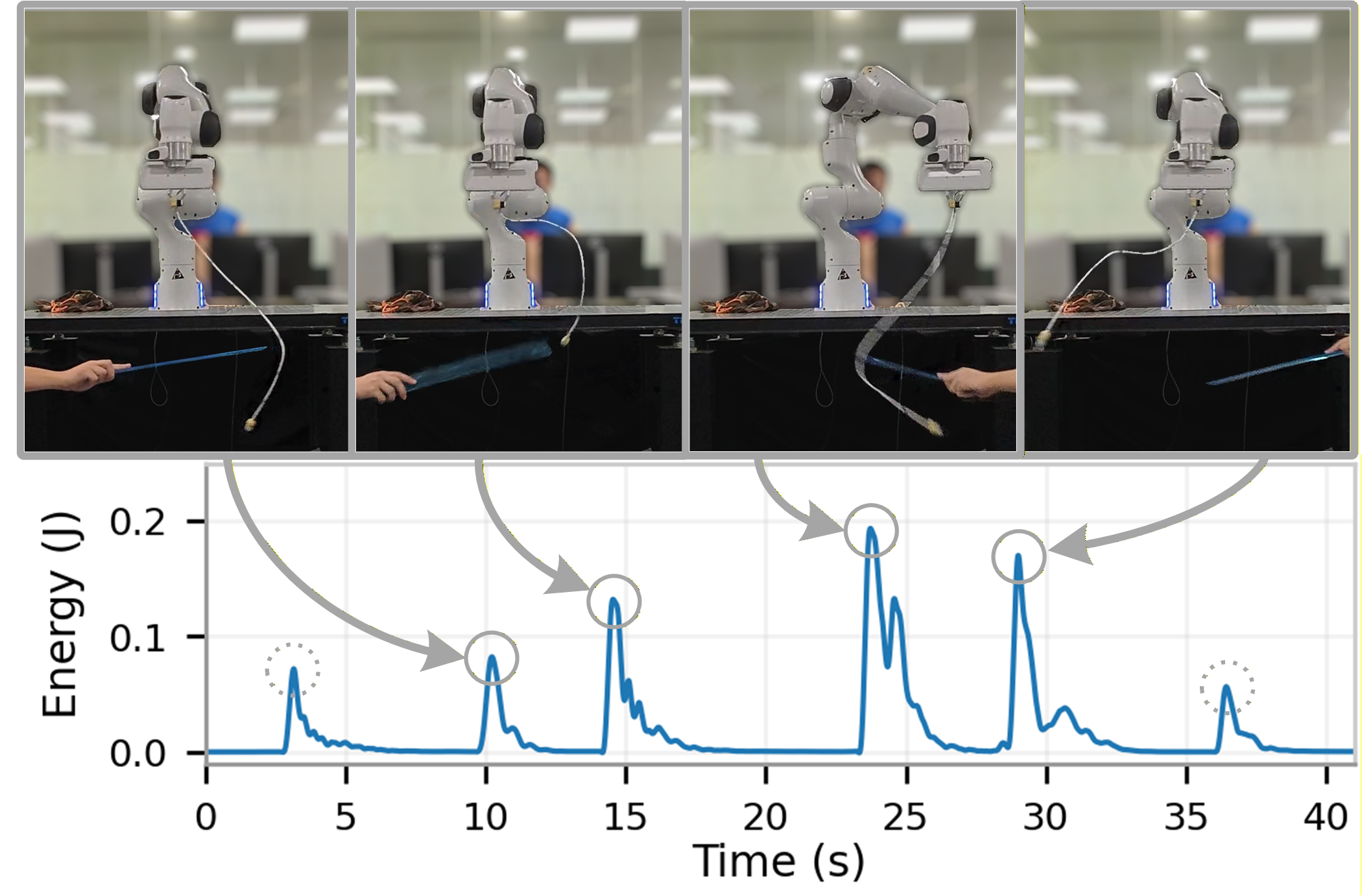}
    \caption{\textbf{Robustness to external disturbances.} The rope is manually disturbed to evaluate system robustness. The corresponding energy profile and key snapshots of the experiment are presented. The system rapidly dissipated the injected energy and achieved stability.}
    \label{fig:disturbance}
    \vspace{-0.4cm}
\end{figure}

\textbf{Robustness to external disturbances.}
To evaluate the robustness of SPiD, we manually disturb the base-rope at arbitrary points and \highlight{arbitrary 3D} directions.
Upon the application of external disturbances, the rope state rapidly transitions from a static condition to fast motion. The controller responds promptly and re-stabilizes the rope.
\cref{fig:disturbance} shows key snapshots of the experiment.
From the energy profile, it can be observed that our controller rapidly dissipates injected energy and achieves fast re-stabilization of the system.

\subsection{Markerless rope stabilization}\label{sec:exp-markerless}

We evaluate our markerless rope perception method in three sections.
We start by presenting our preliminary evaluation of the pretrained segmentation models in Sec.~\ref{sec:exp-markerless-segmentation}. Then, we evaluate performance and limitations of the markerless perception method on our main task, firstly, by comparing it to the marker-based perception, and secondly, by running a rope stabilization control experiment without markers.

\subsubsection{Rope detection and segmentation} \label{sec:exp-markerless-segmentation} 

In this section we compare different detection and segmentation approaches for rope tracking, as introduced in Sec.~\ref{sec:markerless-method}.
We compare three approaches for this problem: \begin{enumerate*}
    \item \textit{Grounded-SAM} combines Grounding DINO \cite{liu2024grounding} for detecting object bounding box (prompt: ``rope'') and SAM \cite{Kirillov_2023_ICCV} for pixel-wise segmentation;
    \item \textit{Grounded-Mobile-SAM} combines Grounding DINO with Mobile-SAM~\cite{zhang2023faster} that replaces SAM’s large image encoder with a lightweight encoder designed for faster inference;
    \item \textit{Fine-tuned YOLOv11-seg} \cite{ultralyticsyolo} detects the object and segments it through a unified model that is fine-tuned with our dataset.
\end{enumerate*}

\textbf{Data collection and fine-tuning. }A rope dataset was collected using an Intel RealSense D435i RGB-D camera. RGB images were captured at 60~FPS to cover a wide range of dynamic rope configurations.
The dataset includes four different ropes with varying colors (blue, green, white, and black), lengths, and thicknesses. Data were recorded under diverse conditions, including multiple backgrounds, illumination variations, and dynamic motions, to improve robustness and generalization.
The final dataset consists of 12{,}400 images for training and 2{,}200 images for validation and testing.
To better capture the thin ropes, high resolution images ($960 \times 960$  px) were used during both training and inference.
Pixel-level annotations were initially generated by Grounded-SAM's 
text-prompt-based detection and segmentation. 
These masks were then manually reviewed and refined 
using a custom graphical annotation interface and 
the CVAT\footnote{CVAT: Computer Vision Annotation Tool, https://www.cvat.ai} annotation tool.

The YOLOv11-seg model is fine-tuned for 100 epochs 
with cosine learning-rate decay and mild weight regularization to improve segmentation stability. 
During inference, each RGB frame is resized to the network input resolution and processed in real time on an NVIDIA RTX 2000 Ada Generation GPU with 8 GB of memory.

\textbf{Results. } Table~\ref{tab:quantitative_results} presents the precision, recall, and F1-score and average frames per second (FPS) of each approach on the same test split of the collected dataset 
 under identical hardware conditions.

\begin{table}[t]
\centering
\caption{Quantitative comparison of rope segmentation performance and runtime.}
\label{tab:quantitative_results}
\begin{tabular}{lcccc}
\hline
\textbf{Method} & \textbf{Precision} & \textbf{Recall} & \textbf{F1-score} & \textbf{FPS} \\
\hline
Grounded-SAM & \textbf{0.8703} & 0.8819 & \textbf{0.8743} & 0.41 \\
Grounded-Mobile-SAM & 0.6001 & 0.8727 & 0.6512 & 2.45 \\
Fine-tuned YOLOv11-seg & \textbf{0.8605} & \textbf{0.9434} & \textbf{0.8980} &\textbf{ 39.80} \\
\hline
\end{tabular}
\end{table}

The experimental results highlight the limitations of prompt-based foundation models for real-time rope perception and demonstrate the advantages of a task-specific, fine-tuned segmentation model.
Although Grounded-Mobile-SAM improves inference speed through a lightweight image encoder, the end-to-end latency is still dominated by GroundingDINO, resulting in insufficient frame rates for real-time deployment.
Furthermore, it achieves lower precision and F1-score due to reduced model capacity.
The fine-tuned YOLOv11-seg model overcomes these limitations by jointly performing detection and segmentation in a single forward pass. This unified architecture significantly reduces inference latency while maintaining high segmentation accuracy, with the additional cost of data collection and fine-tuning. 

\begin{figure}[t]
    \centering
    \begin{tabular}{c}
        \includegraphics[width=0.90\linewidth]{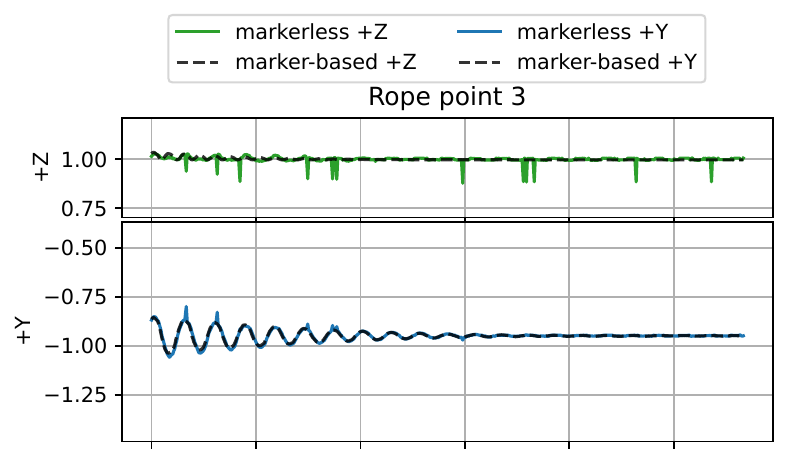} \vspace{-0.3cm}\\
        \includegraphics[width=0.90\linewidth]{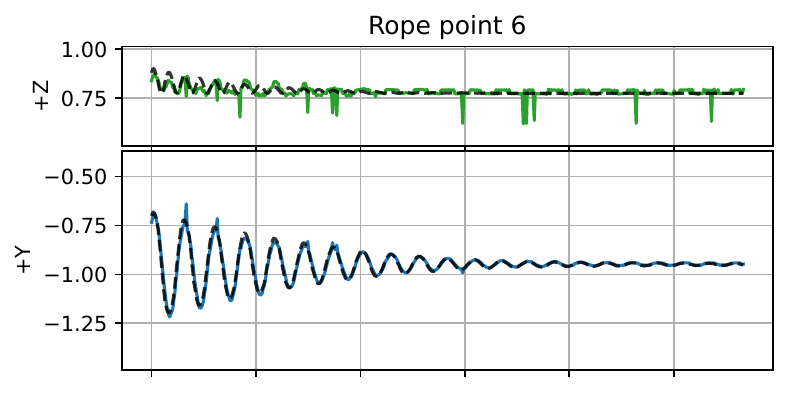} \vspace{-0.3cm}\\
        \includegraphics[width=0.90\linewidth]{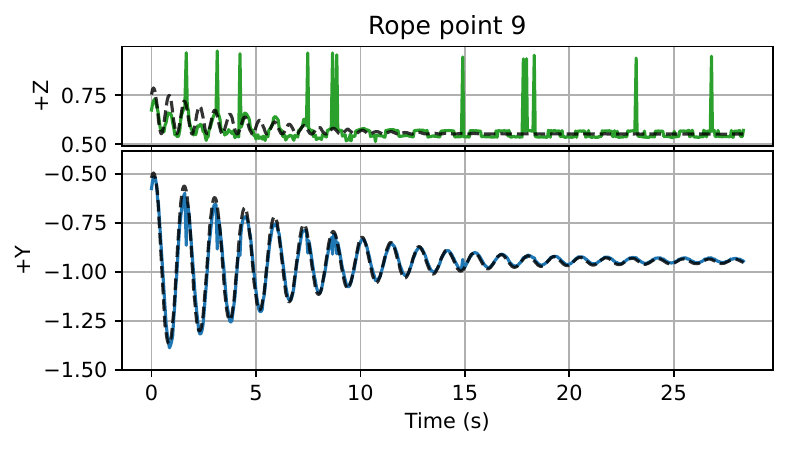} \vspace{-0.3cm}\\
    \end{tabular}
    \caption{Comparison of markerless and marker-based perception, shown for rope points 3, 6 and 9 on \textit{YZ-}axes.}
    \label{fig:markerless-accuracy}
    \vspace{-0.4cm}
\end{figure}

\subsubsection{Comparison to marker-based perception}

In this experiment, we let a rope swing for $28$ \textit{s} while tracking it by both the marker-based and markerless perception systems that are described in Section~\ref{sec:perception}.
We use the marker-based tracking positions as our reference to assess the accuracy of the markerless tracking method.
Marker-based tracking is based on OptiTrack system that uses multiple high-frequency cameras and reflective markers to achieve high accuracy tracking. 
Markerless tracking is done using a single Intel RealSense camera that is more affordable and portable, however, it produces less accurate results at a lower frequency.

The whole markerless perception pipeline, including the segmentation and point splitting, publishes point clouds at $22$ \textit{Hz} frequency on average, as compared to $100$ \textit{Hz} of marker-based perception. After orienting the reference frames by manual calibration and removing mean bias for each axis, the RMSE between the markerless and marker-based tracking over all points and \textit{YZ-}axes is $2.381$ \textit{cm}. We plot the tracked positions of 3 rope points in Fig.~\ref{fig:markerless-accuracy}. As seen on the figure, z-axis (altitude) contains more error. This stems from the momentary shifts that are caused by jumps in rope splitting. The y-axis shows higher accuracy, enabling the rope stabilization control.

\highlight{Unlike the marker-based approach, which reconstructs 3D rope positions, the markerless system tracks rope segments in 2D in the present experiments. 
We also tested 3D markerless tracking by mapping the segmented rope points to the \mbox{RGB-D} camera depth image and recovering their 3D coordinates through frame transformations, together with depth filtering. 
However, the resulting depth estimates were too noisy for reliable control, mainly because the rope is too thin to be captured accurately by the sensor at the operating distance. 
Thus, this limitation is specific to the sensing hardware used in the present setup rather than to the control framework itself.}

\begin{figure}[t]
    \centering
    \includegraphics[width=0.90\linewidth]{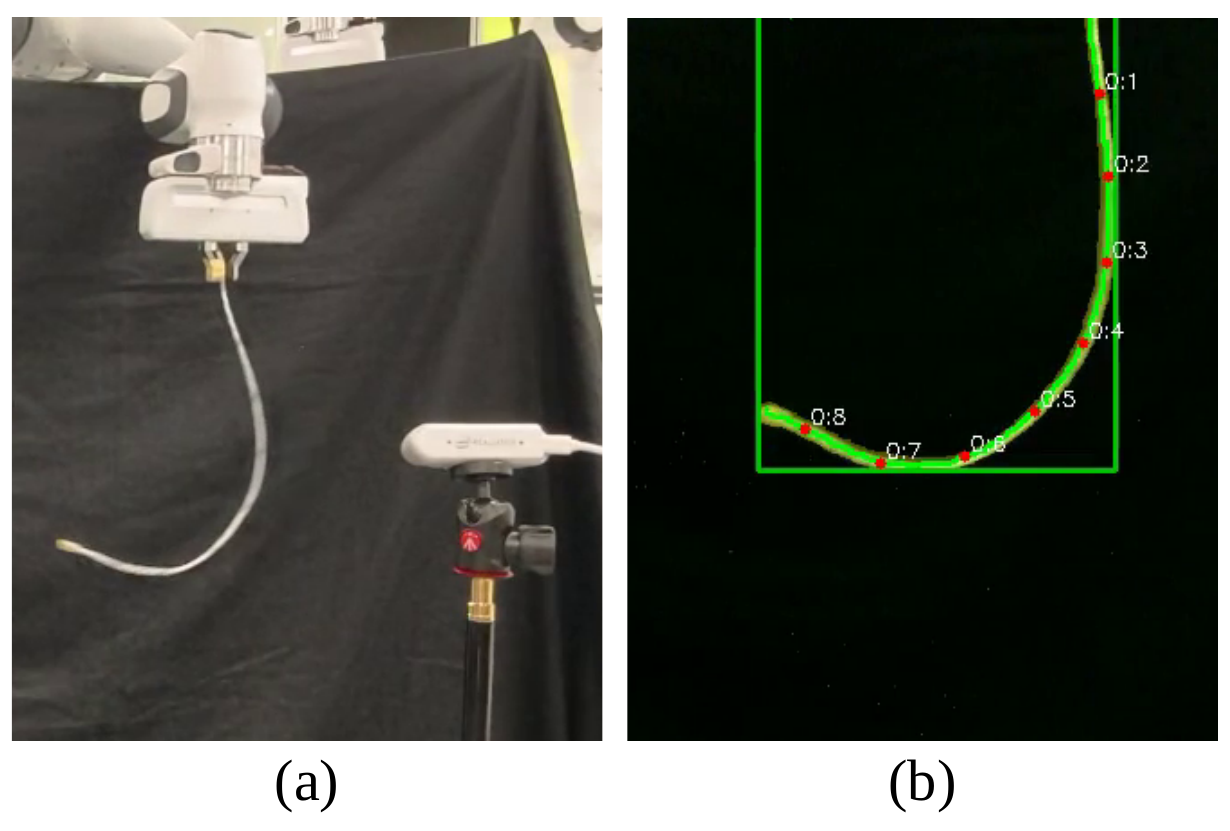}
    \vspace{-0.1cm}
    \caption{\textbf{Markerless rope stabilization.} (a) Experiment setup with a single RGB camera. (b) Annotated camera view with the detection box, segmentation mask and equally split nodes. }
    \label{fig:markerless-setup}
    \vspace{-0.4cm}
\end{figure}

\begin{figure}[t]
    \centering
    \includegraphics[width=1.0\linewidth]{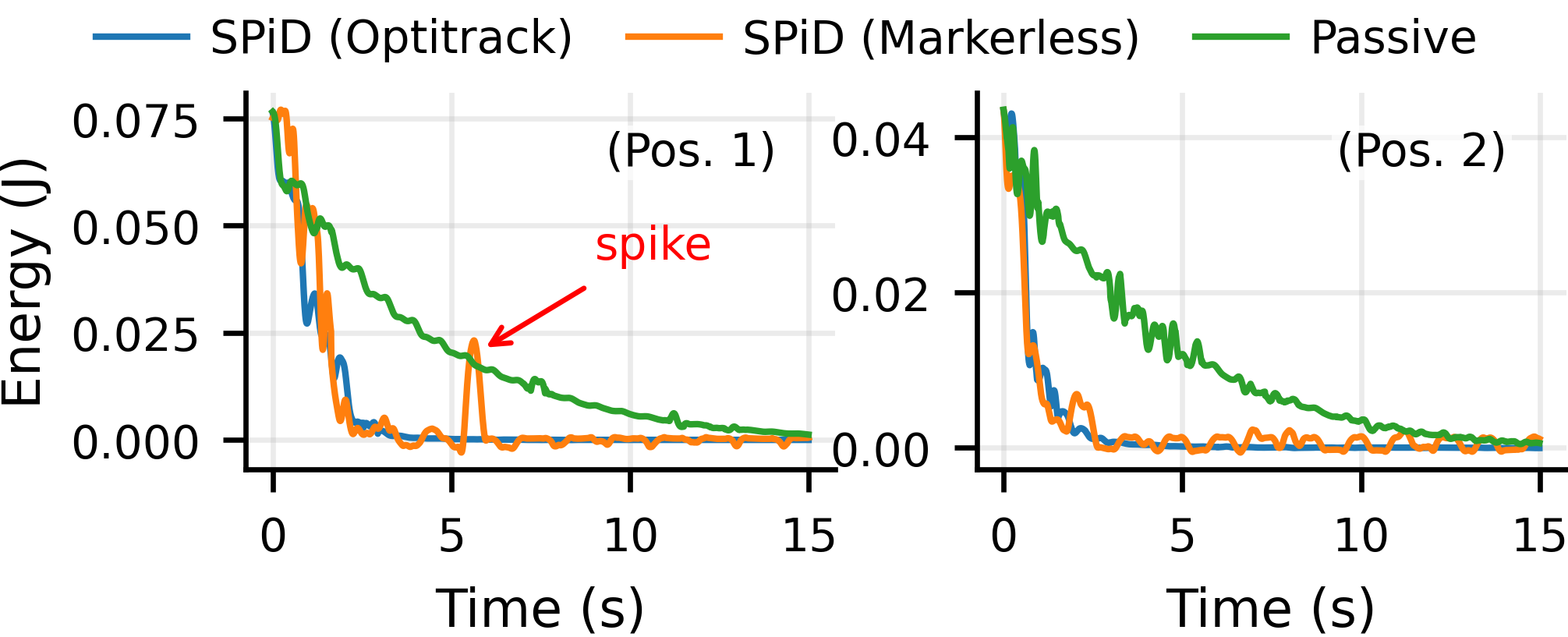}
    \caption{\textbf{Comparison of control performance under markerless perception.}  SPiD maintains similar performance under coarse markerless perception. The oscillations are due to the perception rather than real rope motion, including the annotated spike.}
    \label{fig:markerless-energy-plot}
    \vspace{-0.4cm}
\end{figure}

\subsubsection{Control experiment}
We validate the feasibility of our controller in a markerless single camera setup as shown in Fig.~\ref{fig:markerless-setup}. We used a simple background as the segmentation method can easily confuse parts of the background with a white rope. Another limitation of this approach is the field-of-view. As the camera goes farther than the rope, it becomes harder to segment a thin rope. On the other hand, getting closer constrains the area in which the rope can be observed. 

\begin{figure*}[h]
    \centering
    \includegraphics[width=1.0\textwidth]{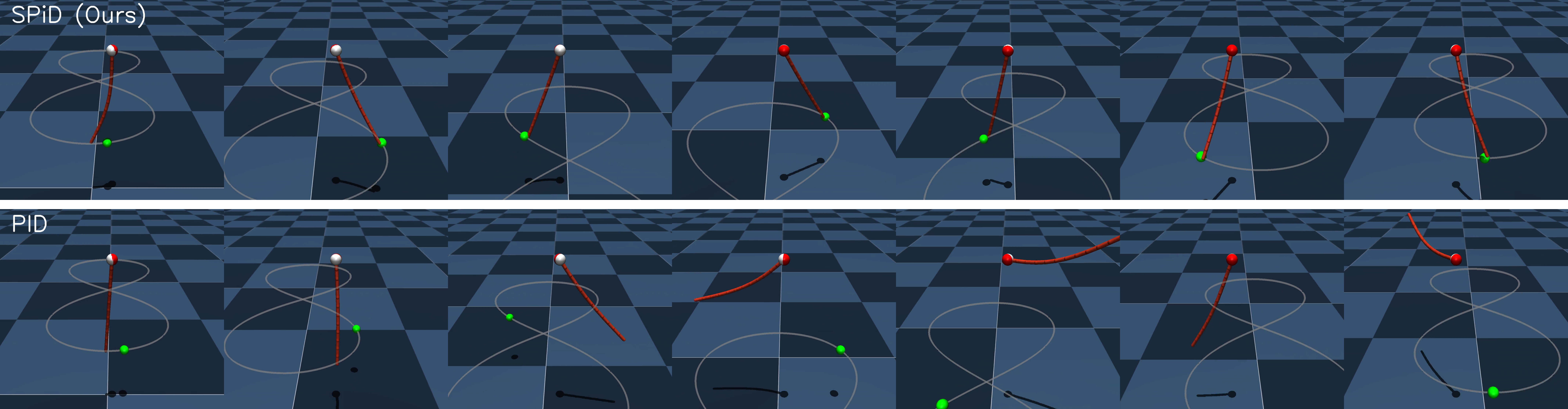}
    \vspace{-0.4cm}
    \caption{\textbf{Snapshots of rope trajectory tracking on the 2.5s Lemniscate.} Top: SPiD (ours); Bottom: PID baseline.}
    \label{fig:trajectory-snapshots}
\end{figure*}

\begin{table*}[t]
\centering
\caption{Mean and standard deviation of tracking error (in meters) over different target trajectories and durations.}
\label{tab:tracking_error}
\begin{tabular}{lcccccc}
\toprule
\multirow{2}{*}{Method} 
& \multicolumn{2}{c}{\textbf{Sinusoid}} 
& \multicolumn{2}{c}{\textbf{Egg}} 
& \multicolumn{2}{c}{\textbf{Lemniscate}} \\
\cmidrule(lr){2-3}
\cmidrule(lr){4-5}
\cmidrule(lr){6-7}
& $T=2.5\,\mathrm{s}$ & $T=5.0\,\mathrm{s}$ 
& $T=2.5\,\mathrm{s}$ & $T=5.0\,\mathrm{s}$ 
& $T=2.5\,\mathrm{s}$ & $T=5.0\,\mathrm{s}$ \\
\midrule
PID  
& $0.1617 \pm 0.0006$ & $0.0398 \pm 0.0014$
& $0.1470 \pm 0.0023$ & $0.0392 \pm 0.0012$
& $0.3082 \pm 0.0005$ & $0.0984 \pm 0.0020$ \\

SPiD (Ours) 
& $\mathbf{0.0299 \pm 0.0004}$ & $\mathbf{0.0134 \pm 0.0003}$
& $\mathbf{0.0275 \pm 0.0004}$ & $\mathbf{0.0109 \pm 0.0002}$
& $\mathbf{0.0221 \pm 0.0005}$ & $\mathbf{0.0127 \pm 0.0001}$ \\
\bottomrule
\end{tabular}
\end{table*}

\begin{figure}[!h]
    \centering
    \includegraphics[width=0.90\linewidth]{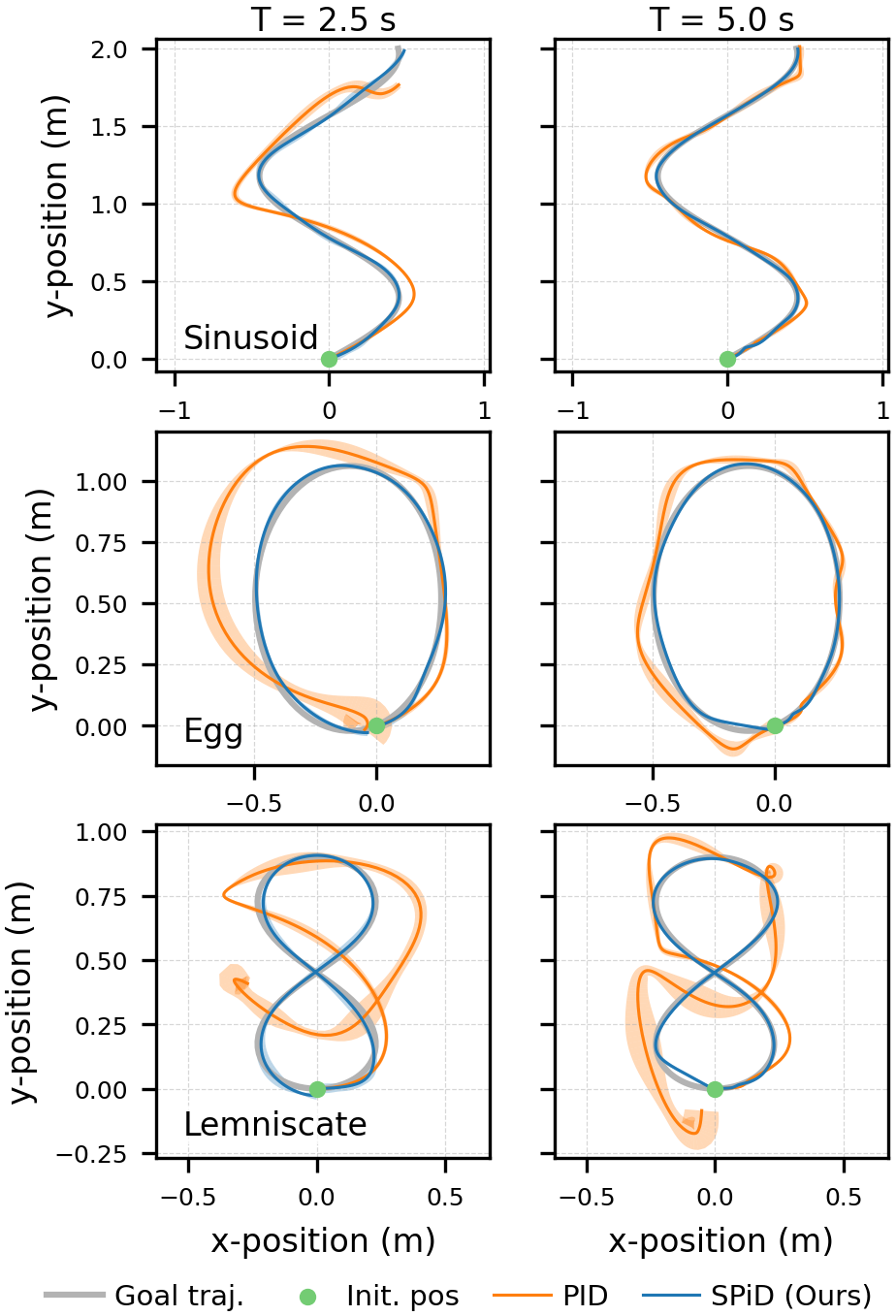}
    \caption{\textbf{Trajectory tracking results.} Three target trajectories under two different durations are evaluated. For each setting, results are averaged over $5^3$ rope models with different physical parameters, with solid lines showing the mean tracked trajectories and shaded regions indicating $\pm 3$ std measured orthogonally to the mean trajectory. Our method achieves more accurate tracking compared to the baseline.}
    \label{fig:trajectory-plots}
    \vspace{-0.4cm}
\end{figure}

We conducted the stabilization experiments with \textit{short-rope} for two initial positions (Fig.~\ref{fig:comparison}). We present the energy plot of \textit{markerless SPiD} in Fig.~\ref{fig:markerless-energy-plot}, comparing it to the results obtained in the previous experiment.
Please note that the remaining spikes in the energy plot are due to the spikes in perception output as seen in Fig.~\ref{fig:markerless-accuracy}, rather than rope motion.
Despite the limitations of the perception system, we were able to achieve equally swift rope stabilization without fine-tuning our controller. 
This is made possible by the robustness of our controller and a fairly accurate rope perception method.

\subsection{Generalization to rope trajectory tracking}\label{sec:exp-trajectory}

To demonstrate the generality of SPiD, we apply the same approach to a new task in simulation. The rope trajectory tracking task requires manipulating the top end of the rope such that its bottom end follows a desired trajectory.
Compared to the rope stabilization task, only two aspects need to be reformulated:
\begin{itemize}
    \item The loss function is changed from the rope energy to the distance between the rope tip trajectory and the target trajectory.
    \item A short segment of the future target trajectory is included in the controller input, so that it learns to act in a predictive manner.
\end{itemize}

\subsubsection{Task formulation}\hfill\break\noindent
\indent\textbf{Task specification.} In this task, the task specification $\bm{G}$ includes the target trajectory, because this information is necessary to define the loss function and achieve predictive behavior. The target trajectory is represented as $\bm{G}=\{ \bm{g}_{\text{goal}}^{t}\}_{t=1...T}$, where $\bm{g_{\text{goal}}}^t \in \mathbb{R}^3$ denotes the target position at time step $t$. Instead of feeding the entire trajectory $\bm{G}$ into the controller, we concatenate the current rope state with the current target point and the next 40 target points.

\indent\textbf{Loss function.}
Given the initial state $\bm{X}^0$, the rope evolves over time under the action of the controller, producing a sequence of states $\{\bm{X}^t\}_{t=1...T}$. 
In the trajectory tracking task, we only require the rope tip to stably track the target trajectory, without considering the trajectories of the other rope nodes. Therefore, we extract the position of the rope tip $\bm{p}_N^t$ from $\bm{X}^t$ and define the tracking loss based on its distance to the target position $\bm{g_{\text{goal}}}^t$ as
\begin{equation}
\begin{split}
\bm{L_\text{task}}\!\left(\{\bm{X}^t\}_{t=1}^{T}, \bm{G} \right)
= \sum_{t=1}^{T} \gamma^{T-t} \cdot \left\| \bm{p}_N^t - \bm{g_{\text{goal}}}^t \right\|^2_2,
\end{split}
\label{eq:loss_track}
\end{equation}
where $\gamma \in (0,1]$ is a discount factor that places greater emphasis on tracking accuracy at later time steps.

\textbf{Metrics.}
We evaluate the tracking performance by the average Euclidean distance between the executed trajectory and the target trajectory.
Specifically, the tracking error is defined as
\begin{equation}
e_{\text{track}} = \frac{1}{T} \sum_{t=1}^{T} \left\| \bm{p}_N^t - \bm{g_{\text{goal}}}^t \right\|_2 ,
\end{equation}
where a smaller value indicates better tracking accuracy.

\subsubsection{Baseline controller}
For comparison, we adopt a PID controller as the baseline.
The control error is defined as the positional deviation between the rope tip and the target point.
For each target trajectory and each trajectory duration, a separate set of PID gains is independently tuned, resulting in a total of six parameter sets.
In contrast, the SPiD employs a single fixed policy across all trajectories and durations.

\subsubsection{Simulation experiments}\hfill\break\noindent
\indent\textbf{Setup.} The simulation experiments are conducted in MuJoCo using a $30\,\mathrm{cm}$ rope model composed of 12 serially connected rigid segments.
We reuse the previously identified physics model to train the controller, while reducing the number of mass points $N$ to 7.
The rope is manipulated by controlling the velocity of its top end in the 3D space at a control frequency of $100\,\mathrm{Hz}$.
We evaluate the performance of the proposed method on three different target trajectories, namely Sinusoid, Egg, and Lemniscate trajectories, as illustrated in \cref{fig:trajectory-plots}.
The total path lengths of the three target trajectories are $3.13\,\mathrm{m}$, $2.92\,\mathrm{m}$, and $2.77\,\mathrm{m}$, respectively.

To examine tracking performance under different motion speeds, we consider two trajectory durations, $2.5\,\mathrm{s}$ and $5.0\,\mathrm{s}$.
For the shorter duration of $2.5\,\mathrm{s}$, the maximum velocities along the sinusoid, egg-shaped, and lemniscate trajectories reach $1.98\,\mathrm{m/s}$, $1.79\,\mathrm{m/s}$, and $1.47\,\mathrm{m/s}$, respectively. Furthermore, following the same setup as in the rope stabilization task, we sample $5^3$ rope models with different physical properties to evaluate the generalization aspect.

\textbf{Results.}
The tracking results are shown in \cref{fig:trajectory-plots}.
Our method consistently achieves accurate tracking performance across different trajectories, time durations, and rope properties. 
In contrast, the baseline method only achieves reasonable performance on slow ($T=5.0$) and simple trajectories (e.g., the sinusoid-egg case), fails to accurately track the Lemniscate trajectory, and exhibits significantly larger tracking variance across different rope properties in all cases.
\cref{fig:trajectory-snapshots} further presents representative tracking snapshots for the $2.5\,\mathrm{s}$ Lemniscate trajectory. 
Finally, the quantitative results are summarized in \cref{tab:tracking_error}, where our method achieves substantially lower mean error and standard deviation in all cases compared with the baseline method.

\highlight{\section{Discussion and Limitations}}

\highlight{While the proposed framework demonstrates strong performance on dynamic rope manipulation, the present study remains subject to several limitations related to modeling scope, parameter interpretation, and perception. We discuss these limitations below to clarify the current scope of the work and to outline directions for future research.}

\highlight{\textbf{Scope of DLO manipulation.}}
\highlight{The current study focuses on DLOs and evaluates the proposed framework on dynamic rope stabilization and trajectory tracking tasks. 
Accordingly, the current model is validated only on objects with a one-dimensional topology and has not yet been evaluated on higher-dimensional deformable objects, such as cloth or sponge-like materials. 
In addition, the present formulation does not explicitly model environmental contacts, such as rope-ground or rope-obstacle friction and collision. Therefore, the current results are limited to free-space dynamic manipulation and do not yet cover contact-rich or cluttered environments. 
Future work will focus on extending the framework to higher-dimensional deformable objects and on incorporating explicit contact modeling.}


\highlight{\textbf{Interpretation of identified model parameters.}}
\highlight{The identified parameters have clear physical roles within the proposed mass--spring--damper formulation, but they should be interpreted as effective parameters of the discretized model rather than direct material-property measurements. 
The model approximates real rope dynamics, and some parameters, particularly bending stiffness, torsional stiffness, and damping coefficients, do not correspond one-to-one with standard material constants; similarly, the identified spatial parameter distribution should be interpreted as an effective approximation of non-uniform rope behavior rather than an exact localization of knots or heavy sections.
Accordingly, the goal of system identification is not material metrology, but obtaining a physically meaningful and dynamically accurate model for prediction and control. 
As shown in the model validation experiments, the identified model reproduces rope motion over a finite prediction horizon with good accuracy.}

\highlight{We also note that the current experiments do not support a claim of generalization to arbitrary DLO materials. Rather, the reported robustness is limited to the range of rope-property variations evaluated in our simulation and real-world experiments.}

\section{CONCLUSIONS}

This paper proposed SPiD, a physics-informed self-supervised learning framework for dynamic manipulation of deformable linear objects.
The effectiveness, generalizability, and robustness of SPiD were demonstrated on an online rope stabilization task through extensive experiments in both simulation and the real world, even under coarse markerless perception.
Furthermore, the same framework was extended to a trajectory tracking task with only minimal modifications.
Across both tasks, the proposed method achieved superior performance over baseline approaches through improved physical modeling and the augmented self-supervised training strategy.
Our dynamics model demonstrated increased accuracy in rope motion prediction with the contribution of the newly introduced damping terms.
Furthermore, it enabled the real-time rope manipulation with a neural network controller.
We were able to train a highly robust and generalizable controller with limited data without expert demonstrations, owing to the proposed augmented self-supervised learning strategy, including a novel self-supervised DAgger variant. 

\balance
\bibliographystyle{IEEEtran}
\bibliography{ref}   

\end{document}